\documentclass[10pt,twocolumn,letterpaper]{article}

\usepackage[pagenumbers]{cvpr} % To force page numbers, e.g. for an arXiv version

\usepackage[utf8]{inputenc}
\usepackage[T1]{fontenc}
\usepackage{nicefrac}
\usepackage{microtype}
\usepackage{tabularx}
\usepackage{multirow}
\usepackage{amsthm}
\usepackage{mathtools}
\usepackage{adjustbox}
\usepackage{wrapfig}

\newcolumntype{Y}{>{\centering\arraybackslash}X}
\usepackage[textsize=tiny]{todonotes}

\usepackage[pagebackref,breaklinks,colorlinks,allcolors=cvprblue]{hyperref}
\usepackage[nameinlink,capitalize]{cleveref}

\newcommand{\answerTODO}[1][]{\textcolor{red}{\bf [TODO]}}
\newtheorem{innercustomthm}{Theorem}
\newenvironment{customthm}[1]
  {\renewcommand\theinnercustomthm{#1}\innercustomthm}
  {\endinnercustomthm}

\newtheorem{innercustomlemma}{Lemma}
\newenvironment{customlemma}[1]
  {\renewcommand\theinnercustomlemma{#1}\innercustomlemma}
  {\endinnercustomlemma}

\theoremstyle{plain}
\newtheorem{theorem}{Theorem}[section]

\newtheorem{lemma}[theorem]{Lemma}

\theoremstyle{definition}
\newtheorem{definition}[theorem]{Definition}

\theoremstyle{remark}

\usepackage[textsize=tiny]{todonotes}

\definecolor{cvprblue}{rgb}{0.21,0.49,0.74}
\usepackage[pagebackref,breaklinks,colorlinks,allcolors=cvprblue]{hyperref}

\def\paperID{*****} % *** Enter the Paper ID here
\def\confName{CVPR}
\def\confYear{2026}

\title{Contrastive Attention Mitigates Spectral Bias in Spiking Transformers}

\author{%
Xiaoli Liu \textsuperscript{1} \quad\quad
Malu Zhang \textsuperscript{1} \thanks{Corresponding author} \quad\quad
Yang Yang \textsuperscript{1} \\
\textsuperscript{1}School of Computer Science and Engineering, \\ University of Electronic Science and Technology of China\\
{\ttfamily\small
liuxiaoli0720@std.uestc.edu.cn
\quad maluzhang@uestc.edu.cn}
}

\begin{document}
\maketitle
\begin{abstract}
Spiking Transformers merge the energy-efficiency of spiking neural networks (SNNs) with the representational power of self-attention, creating a promising architecture for high-performance, energy-efficient computation. However, a performance gap persists versus its counterparts in artificial neural networks (ANNs). Unlike prior works attributing this to binary activations, we reveal that both spiking neurons and spiking self-attention (SSA) act as low-pass filters through multiscale spectral analysis. This characteristic leads to the dissipation of high-frequency components. To address this issue, we propose the Spiking Contrastive Attention (SCA) paradigm, which draw inspiration from the edge-detection and differential sensing properties of biological visual system. By extracting contrast prototypes via global contrastive aggregation and applying local differential refinement, SCA effectively enhances high-frequency information. Extensive experiments show that SCA is a general module that consistently boosts Spiking Transformers across image classification, semantic segmentation, and event-based tracking. Furthermore, it achieves lower complexity, offering superior efficiency over original SSA. These results establish its potential as a fundamental building block for energy-efficient Spiking Transformers.
\end{abstract}    
\section{Introduction}
\label{sec:intro}

Spiking Neural Networks (SNNs) \cite{izhikevich2003simple,maass1997networks} offer an energy-efficient alternative to Artificial Neural Networks (ANNs) \cite{zhouspikformer,yao2023spike}. This efficiency derives primarily from spiking neurons that leverage spatiotemporal dynamics to mimic the computational principles of the human brain \cite{maass1997networks,masquelier2008spike}. Spiking neurons transmit information asynchronously \cite{wei2024event} and consume energy only when receiving or emitting spikes, remaining silent otherwise. This binary activation mechanism allows SNNs to substitute the energy-intensive Multiply-Accumulate (MAC) operations \cite{davies2018loihi,wang2023spatial,zhang2024anp} essential to ANNs with simpler spike-based accumulation. Capitalizing on this energy efficiency, Spiking Transformers integrate the powerful Transformer architecture with SNN characteristics, garnering increasing research attention \cite{yao2025scaling}.

\begin{figure}[t]
    \centering    
    \includegraphics[width=\linewidth]{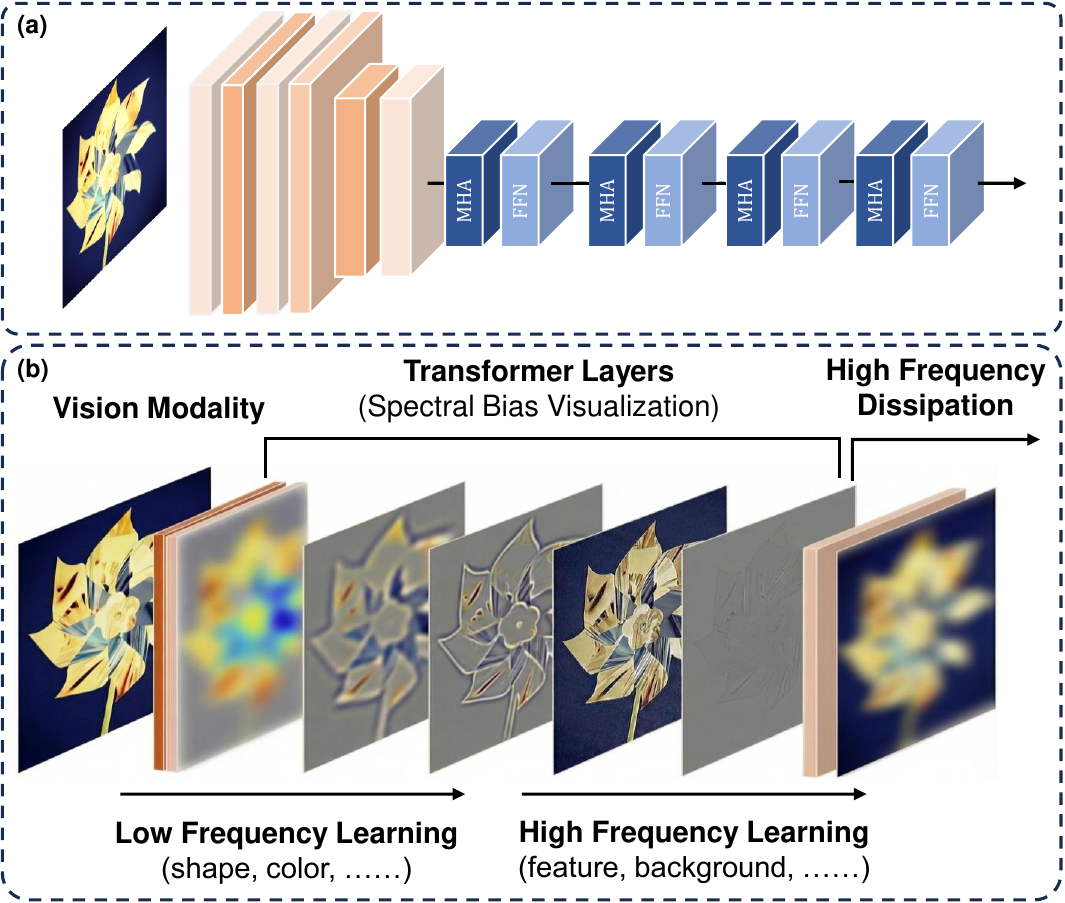}  % 设置页面宽度
    \vspace{-5mm}
    \caption{The description of spectral bias.}
    \label{fig:problem description}
    \vspace{-5mm}
\end{figure}

\begin{figure*}[t]
    \centering    
    \includegraphics[width=\linewidth]{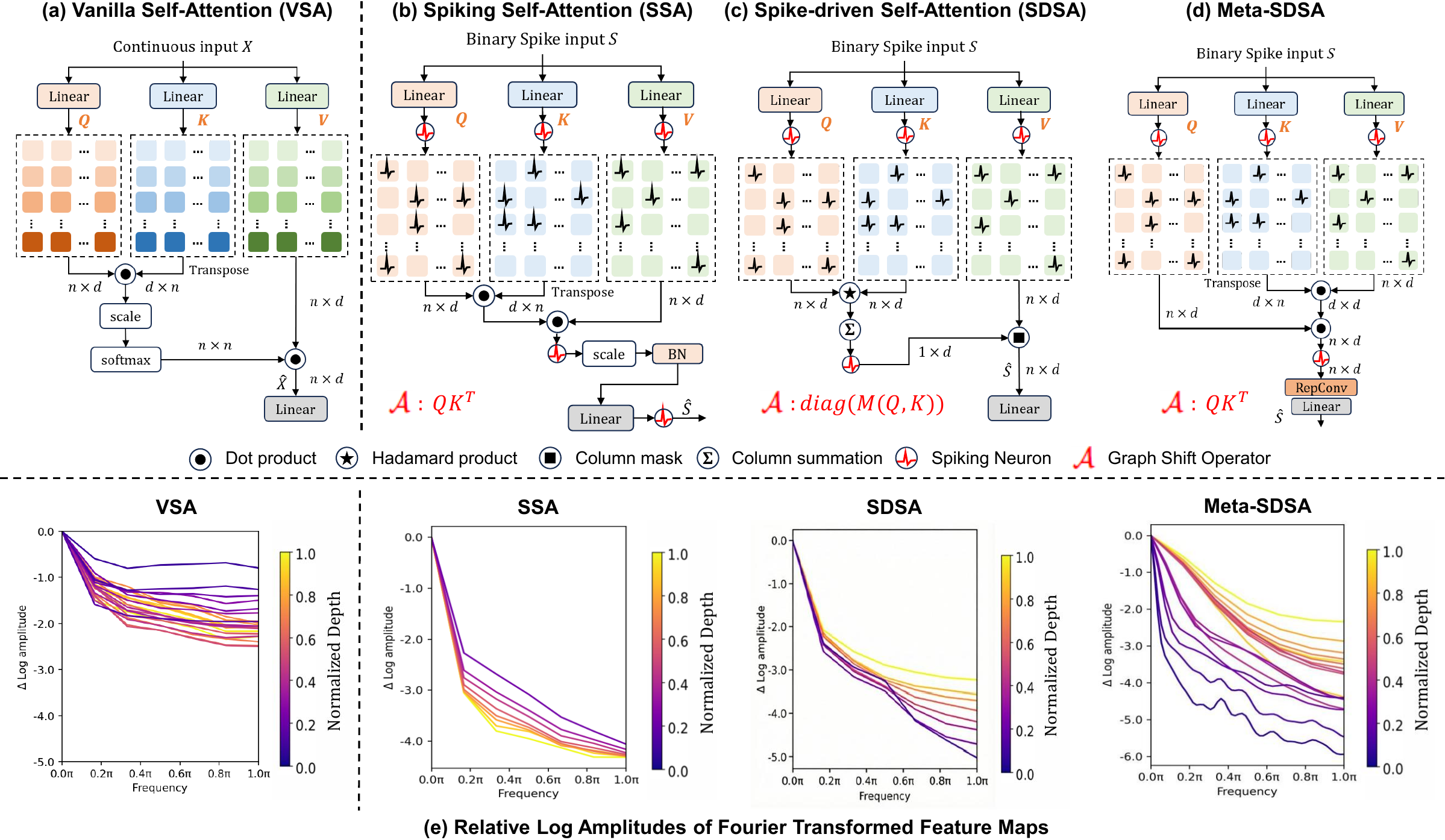}
    \vspace{-5mm}
    \caption{(a) Vanilla Self-Attention (VSA). (b)-(d) Softmax-free Spiking Self-Attention (SSA, SDSA and Meta-SDSA). (e) Relative Logarithmic Amplitude of Fourier Transform Features.}
    \vspace{-5mm}
    \label{fig:attention_fourier}
\end{figure*}

Currently, the primary obstacle for SNNs is the performance gap relative to ANNs \cite{li2021free}. Existing studies typically attribute this bottleneck to the quantization errors arising from binary spike communication \cite{guo2022loss,li2021differentiable,bu2023optimal}. Although strategies based on surrogate gradients \cite{wu2018spatio} and improved spiking neurons \cite{luo2024integer} have mitigated these issues, the representational limitations of SNNs persist.

This work systematically investigates this issue via spectral analysis, identifying \textbf{\textit{Spectral Bias}} \cite{rahaman2019spectral,park2022vision} as the critical factor limiting Spiking Transformers' performance. We analyze spiking neurons and three representative spiking self-attention (SSA) mechanisms, the results reveal that the SSA mechanism, relying on spiking neurons and softmax-free \cite{zhouspikformer,yaospike,yao2023spike}, is equivalent to a low-pass filter (as shown in Fig. \ref{fig:attention_fourier}). This inherent characteristic induces a strong spectral bias during training: the network prioritizes low-frequency components (e.g., contours and backgrounds) while suppressing high-frequency components (e.g., textures, edges, and fine structures), ultimately impairing the model's capacity to capture fine visual features (as shown in Fig. \ref{fig:problem description}).

% To overcome this limitation, we propose a \textbf{Spiking Contrastive Attention (SCA)} paradigm. Diverging from the original SSA mechanism, we partition the feature space into excitatory and inhibitory streams. Through a spike-driven learnable membrane potential fusion strategy, this paradigm dynamically integrates excitatory-inhibitory contrast tokens, effectively simulating the differential perception characteristics of biological visual systems to enhance feature sharpness. This is also consistent with edge-detection mechanisms that enhance high-frequency information. \cite{suzuki2003neural,gonzalez2009digital}. Furthermore, we design a two-stage architecture. This architecture discards the computationally intensive full-rank attention mechanism in favor of extracting contrast tokens as global prototypes. These tokens interact with local queries in a patch-wise manner, enabling the network to efficiently aggregate global context while optimizing local representations through a differential recursive process. The main contributions are summarized as follows:

To overcome this limitation, we propose \textbf{Spiking Contrastive Attention (SCA)}, a two-stage attention mechanism inspired by the edge-detection and differential sensing property of biological visual system \cite{suzuki2003neural,gonzalez2009digital}. In Stage I, SCA constructs complementary excitatory and inhibitory contrast prototypes from spatially compressed query features. These prototypes aggregate global key-value information through two parallel streams, whose responses are adaptively integrated by a differential membrane-potential fusion to generate a compact global contrast map. In Stage II, full-resolution query tokens interact with the contrast prototypes and the global contrast map through a second differential fusion, thereby refining local representations and emphasizing edge-related high-frequency information. By replacing quadratic token interactions with efficient interactions between \(N\) query tokens and \(n\) contrast prototypes (\(n \ll N\)), SCA effectively aggregates global context while reducing the complexity from \(\mathcal{O}(TN^{2}C)\) to \(\mathcal{O}(TNnC)\). The main contributions are summarized as follows:
\begin{itemize}
    \item We presents theoretical analysis of Spiking Transformer in the frequency-domain perspective at both neuronal and architectural levels. We reveal that the spiking neuron and several representative spiking self-attention mechanisms exhibit low-pass characteristics, which leading to severe spectral bias.
    \item To address this issue, we propose Spiking Contrastive Attention (SCA), which inspired by the edge detection and differential perception property of biological visual system. This achieves efficient aggregation of global context, and enhancing the model's ability to represent high-frequency information.
    % \item To address this issue, we propose Spiking Contrastive Attention (SCA). This achieves efficient aggregation of global context, enhancing the model's ability to represent high-frequency information.
    \item Extensive experiments demonstrate that SCA achieves substantial performance improvements across diverse tasks, including image classification, semantic segmentation, and event-based tracking. These results establish its potential as a fundamental building block for Spiking Transformers.
\end{itemize}
\section{Preliminary and Related Works}
\label{main:related works}

\subsection{Spiking Neuron Models}
\label{related works: spiking neuron models}

SNNs achieve spike-driven processing using nonlinear spiking neurons. The Leaky Integrate-and-Fire (LIF) model is a widely adopted mathematical abstract for spike-driven behavior, achieving a balance between biological plausibility and computational efficiency. This work focuses on the Spiking Transformer utilizing LIF neurons as its primary subject of analysis. The dynamic model of LIF neuron is described as:
\begin{align}
    & H[t]=V[t-1]+\frac{1}{\tau}(X[t]-(V[t-1]-V_{reset})),  \label{eq: lif1}\\
    & S[t]=\Theta(H[t]-V_{th}),  \label{eq: lif2}\\
    & V[t]=H[t](1-S[t])+V_{reset}S[t], \label{eq: lif3}
\end{align}
where \(\tau\) is the membrane time constant, and \(X[t]\) is the input current at time step \(t\). When the membrane potential \(H[t]\) exceeds the firing threshold \(V_{th}\), the spike neuron will trigger a spike \(S[t]\). \(\Theta(\cdot)\) is the Heaviside step function which equals \(1\) for \(v \ge 0\) and \(0\) otherwise. \(V[t]\) represents the membrane potential after the trigger event which equals \(H[t]\) if no spike is generated, and otherwise equals to the reset potential \(V_{reset}\).

\subsection{Spiking Transformer}
\label{related works: spiking transformer}

The Vanilla Self-Attention (VSA) serves as the cornerstone of modern Transformers for global information aggregation \cite{vaswani2017attention}, but its computational intensity remains a bottleneck. Recent advancements in Spiking Transformers have demonstrated attractive performance combined with reduced energy consumption. Spikformer \cite{zhouspikformer} pioneered the spike-based self-attention mechanism. Following its success, numerous studies have sought to enhance Spiking Transformers by adopting advanced Transformer architectures \cite{yao2023spike,yaospike} or designing spike coding mechanisms \cite{hao2024lm} or multi-spike neurons \cite{you2024spikezip,yao2025scaling} to minimize representation error. Max-Former \cite{fang2025spiking} posits that high-frequency information loss degrades Spiking Transformer performance and proposes a MetaFormer-style architecture to resolve it. In this work, we conduct frequency-domain analyses at both the neuronal and architectural levels, proposing a plug-and-play module to mitigate the resulting spectral bias. 

% \begin{table}[htbp]
%     \centering
%     \caption{Unified forms of Spiking Self Attention from a GSP perspective.}
%     \label{tab:unified spiking self attention representation}
%     % 1. 缩小字体大小
%     \footnotesize  
%     % 2. 极度压缩列间距 (默认是6pt，这里改为1.5pt)
%     \setlength{\tabcolsep}{6pt} 
%     % 定义表格：
%     % @{} 去掉表格最左和最右的留白
%     % l: 第一列左对齐，宽度自适应内容（最紧凑）
%     % c: 第二列居中，宽度自适应内容
%     % X: 第三列自动填充剩余所有空间
%     \begin{tabularx}{\linewidth}{@{} l c c @{}}
%         \toprule
%         \textbf{Method} & \textbf{Type} & \textbf{Unified Form} \\
%         \midrule
%         % 第一行: SSA
%         SSA & Attention & \(\mathcal{SN}(\mathrm{BN}(\mathrm{Lin}(\mathbf{QK}^\top\mathbf{V}\times s)))\) \\
%         \cite{zhouspikformer} & \(\mathcal{A}\) & \(\mathbf{QK}^\top\) \\
%         \midrule
%         % 第二行: SDSA
%         SDSA & Attention & \(\mathcal{SN}(\sum_c(\mathbf{Q}_S \otimes \mathbf{K}_S))\otimes \mathbf{V}_S\) \\
%          \cite{yao2023spike} & \(\mathcal{A}\) & \(\mathrm{diag}(\mathcal{M}(\mathbf{Q},\mathbf{K}))\) \\
%         \midrule
%         % 第三行: Meta-SDSA
%         Meta-SDSA & Attention & \(\mathcal{SN}_s(\mathbf{Q}_S(\mathbf{K}_S^\top \mathbf{V}_S))\) \\
%         \cite{yaospike} & \(\mathcal{A}\) & \(\mathbf{QK}^\top\) \\
%         \bottomrule
%     \end{tabularx}
% \end{table}
\section{Problem Analysis}
\label{main: theoretical analysis}

This section analyzes the spectral bias of Spiking Transformers from neuronal and architectural perspectives. Section \ref{theoretical analysis: analysis of spiking neuron} shows that LIF charging dynamics inherently exhibit low-pass filtering, while Section \ref{theoretical analysis: analysis of SSA} analyzes representative softmax-free spiking self-attention mechanisms through graph signal processing and reveals their suppression of high-frequency components.

\subsection{Analysis of Spiking Neuron}
\label{theoretical analysis: analysis of spiking neuron}

To analyze the frequency-domain characteristics of the LIF neuron, we first extract its charging process. Based on Equ.\ref{eq: lif1} to Equ.\ref{eq: lif3}, this can be formulated as:
\begin{equation}
    V[t] = \beta V[t-1] + (1 - \beta)X[t],
\end{equation}
where \(\beta = 1 - \frac{1}{\tau} \in (0, 1)\) represents the decay factor, and \(\tau\) is the time constant. \cref{thm: fourier analysis of spiking neuron} provides the frequency-domain characteristics at both the single-neuron and network levels.

\begin{theorem}
\label{thm: fourier analysis of spiking neuron}
The charging process of a single LIF neuron can be equivalently modeled as a first-order Infinite Impulse Response (IIR) low-pass filter, where its squared amplitude-frequency response decreases strictly monotonically with the digital angular frequency \(\omega \in [0, \pi]\). For an \(L\)-layer SNN, assuming the synaptic weight frequency response and the equivalent spike gain of the \(i\)-th layer are denoted by \(W_i(e^{j\omega})\) and \(k_i\) respectively, the overall end-to-end frequency response \(H_L(e^{j\omega})\) is exponentially modulated by the low-pass factor, specifically formulated as (proof in Supp. \ref{appendix: proof of theorem 3.1}):
\begin{equation}
    H_L(e^{j\omega}) = \left(\prod_{i=1}^L k_i W_i(e^{j\omega})\right) \left(\frac{1 - \beta}{1 - \beta e^{-j\omega}}\right)^L.
\end{equation}
\end{theorem}

\subsection{Analysis of Softmax-free Spiking Self-Attention}
\label{theoretical analysis: analysis of SSA}

\textbf{Effect of softmax Operation.} We analyze the effect of the softmax operation from the perspective of Graph Signal Processing (GSP). Within the VSA, we view the token sequence as graph nodes, the attention matrix as the adjacency matrix \(P\), and the feature vectors as signals defined on the graph. Let \(P\) define the connection strengths between nodes. We introduce Dirichlet Energy \cite{zhou2021dirichlet} to quantify the frequency level of signals on the graph.

\begin{definition}
\label{def: laplacian quadratic form}
For a graph signal \(\mathbf{x} \in \mathbb{R}^N\), the Laplacian Quadratic Form based on the normalized adjacency matrix \(P\) is defined as \(E(\mathbf{x})=\mathbf{x}^{T}\mathbf{L}\mathbf{x}=\sum_{i,j}P_{ij}(x_i-x_j)^2\), where \(\mathbf{L}=\mathbf{I}-\mathbf{P}\).
\end{definition}

We utilize \(E(\mathbf{x})\) to characterize signals: \textbf{\textit{a. Low-frequency Signal}}: A small \(E(\mathbf{x})\) implies that signal values \(x_i\) and \(x_j\) at adjacent nodes are similar, indicating global smoothness. \textbf{\textit{b. High-frequency Signal}}: A large \(E(\mathbf{x})\) indicates significant discrepancies between adjacent nodes, corresponding to local sharpness. Let the attention matrix for VSA be defined as \(\mathbf{P}_{\text{soft}}=\text{softmax}(\frac{\mathbf{QK}^T}{\sqrt {d_k}})\), and for softmax-free attention as \(\mathbf{P}_{\text{soft-free}}=\frac{\mathbf{QK}^T}{N}\) (where \(1/N\) serves as the normalization coefficient). We establish the following theorem:

\begin{lemma}
\label{lem: soft-free}
If the feature dimension \(d_k < N\), the Dirichlet Energy of the output signal \(\mathbf{y}_{\text{soft-free}} = \mathbf{P}_{\text{soft-free}} \mathbf{x}\) is constrained by the null space of \(\mathbf{P}_{\text{soft-free}}\), resulting in the truncation of high-frequency components (proof in Supp. \ref{appendix: proof of lemma 3.3}).
\end{lemma}

\begin{theorem}
\label{thm: soft}
The softmax operation employs an element-wise exponential transformation to map the rank-\(d_k\) linear matrix to a full-rank matrix, thereby maximizing the preservation of Dirichlet Energy (proof in Supp. \ref{appendix: proof of theorem 3.4}).
\end{theorem}

Based on Lemma \ref{lem: soft-free}, we can prove that \cref{thm: soft} holds. Consequently, this demonstrates that the softmax operation can, to a certain extent, mitigate the loss of high-frequency information.

% \begin{table}[htbp]
%     \centering
%     \caption{Unified forms of Spiking Self Attention from a GSP perspective.}
%     \label{tab:unified spiking self attention representation}

%     % 1. 缩小字体大小
%     \footnotesize  
%     % 2. 极度压缩列间距 (默认是6pt，这里改为1.5pt)
%     \setlength{\tabcolsep}{6pt} 

%     % 定义表格：
%     % @{} 去掉表格最左和最右的留白
%     % l: 第一列左对齐，宽度自适应内容（最紧凑）
%     % c: 第二列居中，宽度自适应内容
%     % X: 第三列自动填充剩余所有空间
%     \begin{tabularx}{\linewidth}{@{} l c c @{}}
%         \toprule
%         \textbf{Method} & \textbf{Type} & \textbf{Unified Form} \\
%         \midrule
%         % 第一行: SSA
%         SSA & Attention & $\mathcal{SN}(\mathrm{BN}(\mathrm{Lin}(\mathbf{QK}^\top\mathbf{V}\times s)))$ \\
%         \cite{zhouspikformer} & $\mathcal{A}$ & $\mathbf{QK}^\top$ \\
%         \midrule
%         % 第二行: SDSA
%         SDSA & Attention & $\mathcal{SN}(\sum_c(\mathbf{Q}_S \otimes \mathbf{K}_S))\otimes \mathbf{V}_S$ \\
%          \cite{yao2023spike} & $\mathcal{A}$ & $\mathrm{diag}(\mathcal{M}(\mathbf{Q},\mathbf{K}))$ \\
%         \midrule
%         % 第三行: Meta-SDSA
%         Meta-SDSA & Attention & $\mathcal{SN}_s(\mathbf{Q}_S(\mathbf{K}_S^\top \mathbf{V}_S))$ \\
%         \cite{yaospike} & $\mathcal{A}$ & $\mathbf{QK}^\top$ \\
%         \bottomrule
%     \end{tabularx}
% \end{table}

\textbf{Softmax-free Spiking Self-Attention.} Proceeding from the perspective of GSP, we consider the \(N\) tokens in Spiking Transformer as nodes within a graph \(\mathcal{G}=(\mathcal{V}, \mathcal{E})\), treating the input features \(V \in \mathbb{R}^{N \times d}\) as graph signals defined on this domain. 

We focus on three variants of softmax-free spiking self attention that utilize LIF neuron as activation functions. They can be uniformly abstracted as the application of a Graph Shift Operator \(\mathcal{A}\) to the graph signal \(V\): \(\mathbf{O}=\mathcal{A} \cdot V\), where \(\mathcal{A} \in \mathbb{R}^{N \times N}\) denotes the Generalized Adjacency Matrix constructed from \(\mathbf{Q}\) and \(\mathbf{K}\). As illustrated in the Table \ref{tab:unified spiking self attention representation}, distinct attention mechanisms correspond to specific formulations of \(\mathcal{A}\). In the frequency domain, the above equation is equivalent to:
\begin{equation}
    \hat{\mathbf{O}}(\lambda)=\hat{\mathcal{A}}(\lambda) \odot \hat{V}(\lambda),
\end{equation}
where \(\hat{\mathcal{A}}(\lambda)\) denotes the frequency response function of the operator \(\mathcal{A}\).

\begin{theorem}
\label{thm: softmax-free spiking self attention}
SSA \cite{zhouspikformer}, SDSA \cite{yao2023spike}, and Meta-SDSA \cite{yaospike} behave as low-pass filters in the frequency domain (proof in Supp. \ref{appendix: proof of theorem 3.5}).
\end{theorem}

\begin{table}[t]
    \centering
    \caption{Unified forms of SSA from a GSP perspective.}
    \vspace{-3mm}
    \label{tab:unified spiking self attention representation}
    \footnotesize
    \setlength{\tabcolsep}{3pt}
    \resizebox{\columnwidth}{!}{%
        \begin{tabular}{@{} l c l @{}}
            \toprule
            \textbf{Method} & \textbf{Type} & \textbf{Unified Form} \\
            \midrule
            \multirow{2}{*}{SSA \cite{zhouspikformer}}
            & Attn
            & \(\mathcal{SN}(\mathrm{BN}(\mathrm{Lin}(\mathcal{SN}(\mathbf{QK}^\top\mathbf{V}*s))))\) \\
            & \(\mathcal{A}\) & \(\mathbf{QK}^\top\) \\
            \midrule
            \multirow{2}{*}{SDSA \cite{yao2023spike}}
            & Attn
            & \(\mathcal{SN}(\sum_c(\mathbf{Q}_S \otimes \mathbf{K}_S))\otimes \mathbf{V}_S\) \\
            & \(\mathcal{A}\)
            & \(\mathrm{diag}(\mathcal{M}(\mathbf{Q},\mathbf{K}))\) \\
            \midrule
            \multirow{2}{*}{Meta-SDSA \cite{yaospike}}
            & Attn
            & \(\mathcal{SN}_s(\mathbf{Q}_S(\mathbf{K}_S^\top\mathbf{V}_S))\) \\
            & \(\mathcal{A}\) & \(\mathbf{QK}^\top\) \\
            \bottomrule
        \end{tabular}%
    }
    \vspace{-4mm}
\end{table}

\cref{thm: softmax-free spiking self attention} demonstrates that the above spiking self attention behave as low-pass filters. Specifically, \textit{SDSA} employs an element-wise Hadamard product. In the GSP perspective, this corresponds to a diagonal adjacency matrix \(\mathcal{A} = \text{diag}(\mathcal{M}(\mathbf{Q},\mathbf{K}))\). This indicates that \textit{SDSA} treats the token sequence as a set of isolated nodes with self-loops, thereby avoiding the transmission of high-frequency signals between nodes. For \textit{SSA} and \textit{Meta-SDSA}, their operaitions are mathematically equivalent to applying a factorized low-rank adjacency matrix \(\mathcal{A}=\mathbf{Q}\mathbf{K}^T\) to the value signal \(V\). Thus, they share similar low-pass-like characteristics.

\begin{figure*}[t]
    \centering    
    \includegraphics[width=\linewidth]{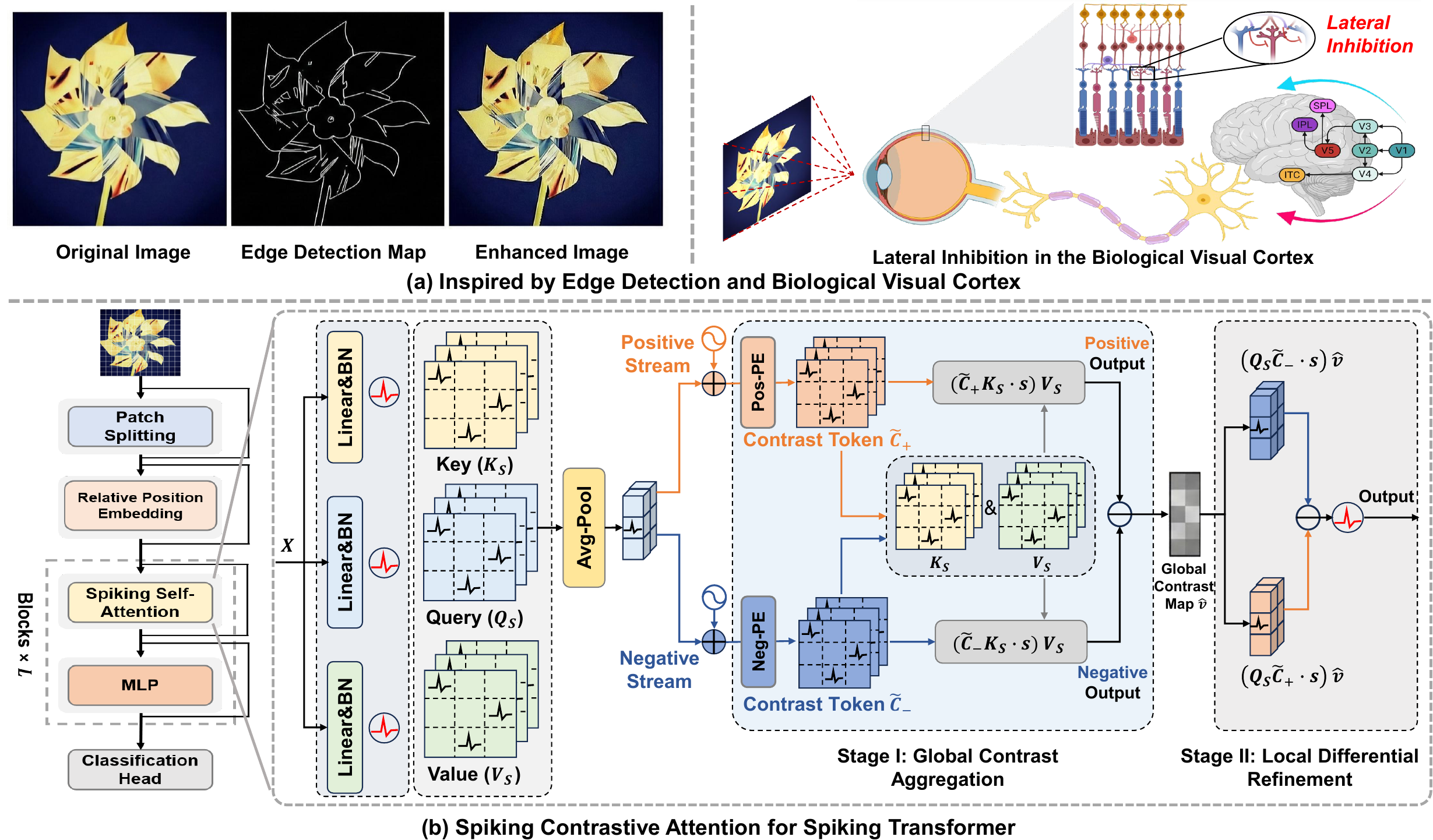}
    \vspace{-5mm}
    \caption{(a) Edge Detection and Biological Visual Cortex. (b) Spiking Contrastive Attention for Spiking Transformer.}
    \vspace{-5mm}
    \label{fig:main figure}
\end{figure*}

\section{Methodology}

This section presents the proposed Spiking Contrastive Attention (SCA) mechanism for mitigating the spectral bias identified above. Section \ref{method: differential} motivates differential processing as a biologically inspired approach to preserving high-frequency information in Spiking Transformers. Section \ref{method: spiking contrastive attention mechanism} describes the two-stage SCA mechanism, which performs global contrast aggregation followed by local differential refinement. Section \ref{method: complexity analysis} analyzes its computational complexity and compares it with vanilla spiking self-attention, demonstrating the efficiency of the proposed prototype-based interactions.

\subsection{Differential for Mitigating High-frequency Losses}
\label{method: differential}

Based on the above analysis, the spectral bias in spiking self-attention stems from the low-pass characteristics of LIF neurons combined with the absence of the sharpening effect induced by softmax. However, the softmax function, which involves exponentiation and division, violates the hardware-efficient computational principles of SNNs \cite{zhou2024spiking}. Consequently, we investigate methodologies to mitigate spectral bias without resorting to softmax.

In biological visual system, the lateral inhibition enables the brain to prioritize salient regions \cite{del2025lateral}. Specifically, the activation of neurons within a focal region induces the suppression of neighboring neurons (Fig. \ref{fig:main figure}(b)), thereby refining visual processing through the inhibition of irrelevant stimuli and the selective enhancement of prominent features \cite{magnuson2024contra,shaw1975retinal}. Therefore, lateral inhibition can be interpreted as a biological differential operator that selectively enhances high-frequency visual information \cite{huang2016impact,cook1998lateral,marr1980theory}. In traditional computer vision, edge detection algorithms (e.g., Sobel \cite{vincent2009descriptive,gao2010improved} or Laplacian \cite{wang2007laplacian} operators) enhance high-frequency image responses by computing the difference between adjacent pixels (as shown in Fig. \ref{fig:main figure} (a)). Furthermore, prior research has demonstrated that differential mechanism is effective in eliminating redundant information \cite{ye2024differential,pu2025linear}. Consequently, we propose incorporating a differential mechanism into spiking self-attention to enhance high-frequency information and amplify attention towards relevant contexts.

\begin{theorem}
\label{thm: differential}
Introducing a differential mechanism into attention enhances high-frequency information. Furthermore, the differential mechanism computes attention scores as the difference between two separate attention maps, thereby amplifying attention to relevant context (proof in Supp. \ref{appendix: proof of theorem 4.1}).
\end{theorem}

\subsection{Spiking Contrastive Attention Mechanism}
\label{method: spiking contrastive attention mechanism}

To mitigate the spectral bias of spiking self-attention, we develop SCA as a two-stage contrastive attention mechanism. It first constructs positive and negative contrast prototypes to aggregate global information, and then uses these prototypes to differentially refine full-resolution query features. This global-to-local design enhances high-frequency representations while maintaining efficient spike-driven computation.

\subsubsection{Stage I: Global Contrast Aggregation}

Given the input spiking features \(X \in \mathbb{R}^{T \times N \times C}\), the query (\(Q\)), key (\(K\)) and value (\(V\)) projections are first computed via learnable weight matrices. These are subsequently transformed into spike sequences through spiking neuron layers:
\begin{equation}
    \mathcal{M}_{\text{S}} = \mathcal{SN}(\text{BN}(XW_\mathcal{M})), \text{where } \mathcal{M} \in \{Q, K, V\},
\end{equation}
where \(Q_S, K_S, V_S \in \mathbb{R}^{T \times N \times D}\). To facilitate multi-head attention computation, these tensors are reshaped into \(Q_S, K_S, V_S \in \mathbb{R}^{T \times B \times M \times N \times D}\), where \(M\) denotes the number of attention heads and \(D\) represents the dimension of each head.

We extract contrast tokens directly from \(Q_S\) via spatial average pooling. The query matrix \(Q_S \in \mathbb{R}^{T \times B \times M \times N \times D}\) is reshaped to restore its 2D spatial structure, denoted as \(Q_S \to \tilde{Q}_S \in \mathbb{R}^{T \times B \times M \times H \times W \times D}\), where \(H \times W=N\). we apply average pooling along the spatial dimensions, with the kernel size and stride selected to downsample the resolution from \(H \times W\) to \(h \times w\):
\begin{equation}
    \tilde{t}_{\text{contrast}} = \text{AvgPool}_{H \times W \to h \times w}(\tilde{Q}_S).
\end{equation}
To bolster the capability of SNNs in capturing high-frequency information—drawing inspiration from differential operators in edge detection—we bifurcate the contrast token branch into positive and negative contrast streams. Concretely, we introduce two distinct learnable positional embeddings, \(e^+, e^- \in \mathbb{R}^{h \times w \times d}\), to construct these respective streams:
\begin{equation}
    \tilde{C}_{+}=\mathcal{SN}(\tilde{t}_{\text{contrast}}+e^{+}), \ \tilde{C}_{-}=\mathcal{SN}(\tilde{t}_{\text{contrast}}+e^-),
\end{equation}
where \(\tilde C_{+}\) and \(\tilde C_{-} \in \{0,1\}^{T \times h \times w \times d}\) represent the excitatory and inhibitory prototypes within the scene, respectively.

To accommodate the characteristics of SNNs, we employ scaled dot-product to modulate the sharpness of the attention distribution and leverage contrast tokens to aggregate global key-value information. Specifically, we compute the accumulated currents for the positive and negative contrast streams:
\begin{equation}
\label{equ: dual-stream currents}
    I^{(1)}_{\text{exc}}=(\tilde C_{+} K^T \cdot s)V, \ I^{(1)}_{\text{inh}}=(\tilde C_{-} K^T \cdot s)V,
\end{equation}
where \(s\) serves as a scaling factor to regulate attention sharpness. By introducing a learnable balancing coefficient , we implement differential fusion at the membrane potential level, allowing the model to dynamically integrate the excitatory current  and the inhibitory current:
\begin{equation}
    U_{\text{mem}}=\lambda_1 I^{(1)}_{\text{exc}}-(1-\lambda_1)I^{(1)}_{\text{inh}}.
\end{equation}
Finally, the membrane potential \(U_{\text{mem}}\) drives the LIF neurons to emit spikes, generating the global contrast map \(\hat{v} = \mathcal{SN}(U_{\text{mem}}) \in \{0, 1\}^{T \times h \times w \times d}\).

\begin{table*}[t]
\centering
\caption{Detailed comparison with other methods on ImageNet-1K.}
\vspace{-3mm}
\label{tab:image classification}
\resizebox{\textwidth}{!}{
\begin{tabular}{lcccccc}
\toprule
\textbf{Method} & \textbf{Type} & \textbf{Architecture} & 
\textbf{\begin{tabular}{@{}c@{}}Time\\ Step\end{tabular}} & 
\textbf{\begin{tabular}{@{}c@{}}Param\\ (M)\end{tabular}} & 
\textbf{\begin{tabular}{@{}c@{}}Energy\\ (mJ)\end{tabular}} & 
\textbf{\begin{tabular}{@{}c@{}}Top-1\\ Acc (\%)\end{tabular}} \\
\midrule
{ViT \cite{vaswani2017attention}}
& ANN & ViT-L/16 & 1 & 304.3 & 80.96 & 79.70 \\
\midrule
{PVT \cite{wang2021pyramid}}
& ANN & PVT-Large & 1 & 61.4 & 45.08 & 81.70 \\
\midrule
\multirow{2}{*}{Spikformer \cite{zhouspikformer}}
& \multirow{2}{*}{SNN} & Spikformer-8-384 & 4 & 16.8 & 7.73 & 70.24 \\
& & Spikformer-8-512 & 4 & 29.7 & 11.57 & 73.38 \\
\midrule
% \multirow{2}{*}{SpikingResformer \cite{shi2024spikingresformer}}
% & \multirow{2}{*}{SNN} & SpikingResformer-S & 4 & 17.8 & 3.37 & 75.95 \\
% & & SpikingResformer-M & 4 & 35.5 & 5.46 & 77.24 \\
% \midrule
\multirow{2}{*}{Spikingformer \cite{zhou2023spikingformer}}
& \multirow{2}{*}{SNN} & Spikingformer-8-512 & 4 & 29.7 & 4.69 & 74.79 \\
& & Spikingformer-8-768 & 4 & 66.4 & 13.68 & 75.85 \\
\midrule
\multirow{2}{*}{Max-Former \cite{fang2025spiking}}
& \multirow{2}{*}{SNN} & Max-10-384 & 4 & 16.23 & 4.89 & 77.82 \\
& & Max-10-512 & 4 & 28.65 &  7.49 & 79.86 \\
\midrule
\multirow{2}{*}{SDT-V1 \cite{yao2023spike}}
& \multirow{2}{*}{SNN} & Spike-driven Transformer-8-384 & 4 & 16.8 & 3.90 & 72.28 \\
& & Spike-driven Transformer-6-512 & 4 & 29.7 & 3.56 & 74.11 \\
\midrule
\multirow{2}{*}{\textbf{SDT-V1+SCA (Ours)}} 
& \multirow{2}{*}{SNN} & Spike-driven Transformer-8-384 & 4 & 16.8 & 3.78 \textbf{\textcolor{violet}{($\downarrow$0.12)}} & \textbf{73.31 \textcolor{violet}{($\uparrow$1.03)}} \\
& & Spike-driven Transformer-6-512 & 4 & 29.7 & 3.40 \textbf{\textcolor{violet}{($\downarrow$0.16)}} & \textbf{75.03 \textcolor{violet}{($\uparrow$0.92)}} \\
\midrule
\multirow{2}{*}{QKFormer \cite{zhou2024qkformer}}
& \multirow{2}{*}{SNN} & HST-10-384 & 4 & 16.47 & 15.13 & 78.80 \\
& & HST-10-512 & 4 & 29.08 & 21.99 & 82.04 \\
\midrule
\multirow{2}{*}{\textbf{QKFormer+SCA (Ours)}} 
& \multirow{2}{*}{SNN} & HST-10-384 & 4 & 16.47 & 15.02 \textbf{\textcolor{violet}{($\downarrow$0.11)}} & \textbf{79.49 \textcolor{violet}{($\uparrow$0.69)}} \\
& & HST-10-512 & 4 & 29.08 & 21.32 \textbf{\textcolor{violet}{($\downarrow$0.67)}} & \textbf{82.77 \textcolor{violet}{($\uparrow$0.73)}} \\
\midrule
\multirow{3}{*}{SDT-V3 \cite{yao2025scaling}}
& \multirow{3}{*}{SNN} & Efficient-transformer-S & 4 & 5.1 & 1.7 & 75.30 \\
& & Efficient-transformer-M & 4 & 10.0 & 3.0 & 78.50 \\
& & Efficient-transformer-L & 4 & 19.0 & 5.9 & 79.80 \\
\midrule
\multirow{3}{*}{\textbf{SDT-V3+SCA (Ours)}} 
& \multirow{3}{*}{SNN} & Efficient-transformer-S & 4 & 5.1 & 1.65 \textbf{\textcolor{violet}{($\downarrow$0.05)}} & \textbf{75.91 \textcolor{violet}{($\uparrow$0.61)}} \\
& & Efficient-transformer-M & 4 & 10.0 & 2.67 \textbf{\textcolor{violet}{($\downarrow$0.33)}} & \textbf{79.13 \textcolor{violet}{($\uparrow$0.63)}} \\
& & Efficient-transformer-L & 4 & 19.0 & 5.66 \textbf{\textcolor{violet}{($\downarrow$0.24)}} & \textbf{80.66 \textcolor{violet}{($\uparrow$0.86)}} \\
\bottomrule
\end{tabular}}
\vspace{-4mm}
\end{table*}

\subsubsection{Stage II: Local Differential Refinement}
The contrast tokens interact with the query tokens to extract features in a differential manner. Specifically, the query tokens \(Q_S\) derive their attention scores using the positive contrast tokens \(\tilde C_{+}\) and negative contrast tokens \(\tilde C_{-}\) within each attention head. In this configuration, the global contrast map \(\tilde v\) functions as the value, while \(\tilde C_{+}\) and \(\tilde C_{-}\) serve as the keys:
\begin{equation}
\label{equ: stage II current}
    I_{\text{exc}}^{(2)} = (Q_S \tilde C_+^T \cdot s) \hat{v}, \quad I_{\text{inh}}^{(2)} = (Q_S \tilde C_-^T \cdot s) \hat{v}.
\end{equation}
Leveraging the differential membrane potential mechanism once more, we compute the membrane potential of the final output neurons and generate the input spikes for the subsequent layer:
\begin{align}
& U_{\text{mem}}^{(2)} = \lambda_2 I_{\text{exc}}^{(2)} - (1 - \lambda_2) I_{\text{inh}}^{(2)},\\
& Y = \mathcal{SN}(U_{\text{mem}}^{(2)}) \in \{0, 1\}^{T \times H \times W \times d}.
\end{align}

\begin{figure*}[htbp]
    \centering    
    \includegraphics[width=\linewidth]{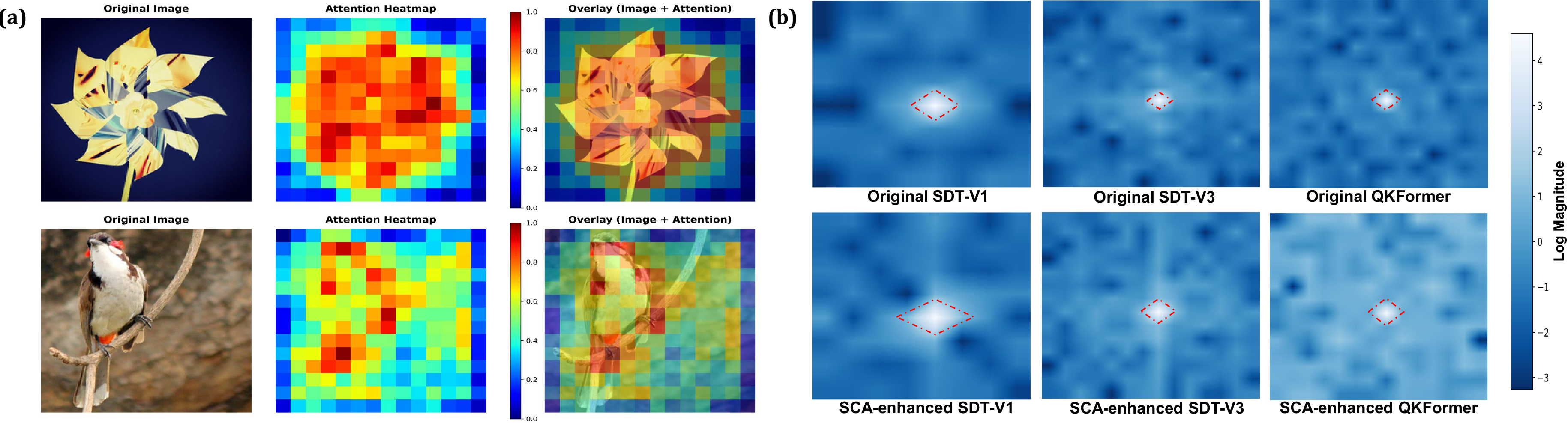}
    \vspace{-4mm}
    \caption{(a) Attention heatmap with SDT-V1+SCA on ImageNet. (b) Fourier spectrum of SDT-V1, SDT-V3, and QKFormer on ImageNet, along with their corresponding SCA-enhanced models.}
    \vspace{-3mm}
    \label{fig:attention map and fourier spectrum}
\end{figure*}

\subsection{Complexity Analysis}
\label{method: complexity analysis}

We retain the notations of Section \ref{method: spiking contrastive attention mechanism}: \(N\) tokens, \(n\) contrast tokens (\(n \ll N\)), \(D\) feature dimension per head, \(M\) heads, \(T\) timesteps and \(C=MD\) channels in total. 

\textbf{Stage I: Global Contrast Aggregation.} According to Eq. \ref{equ: dual-stream currents}, for each head: compute \(\tilde{C}_{+} K^T\), where \(\tilde{C}_{+} \in \mathbb{R}^{n \times D}\) and \(K^T \in \mathbb{R}^{D \times N}\), with a complexity of \(\mathcal{O}(nND)\). Subsequently, multiply the resulting matrix by \(V \in \mathbb{R}^{N \times D}\), which also has a complexity of \(\mathcal{O}(nND)\). Therefore, the computational cost for a single stream is \(2nND\). Given that there are two streams and the computation needs to be accumulated over \(T\) time steps:
\begin{equation}
    \ T \times 2 \times (2nND) = \mathcal{O}(T \cdot NnD).
\end{equation}

\textbf{Stage II: Local Differential Refinement.} Based on Eq. \ref{equ: stage II current}, for each head: compute \(Q_S \tilde{C}_{+}^{T}\). Here, \(Q_S \in \mathbb{R}^{N \times D}\) and \(\tilde{C}_{+}^{T} \in \mathbb{R}^{D \times n}\), with a complexity of \(\mathcal{O}(NnD)\). The resulting attention scores (\(\mathbb{R}^{N \times n}\)) are then multiplied by the global contrastive graph \(\hat{v} \in \mathbb{R}^{n \times D}\), which also has a complexity of \(\mathcal{O}(NnD)\). Therefore, the complexity of Stage II is:
\begin{equation}
    \ T \times 2 \times (2NnD) =  \mathcal{O}(T \cdot NnD).
\end{equation}

The total complexity for \(M\) heads is \(M \times (4T \cdot NnD + 4T \cdot NnD) = \mathcal{O}(T \cdot NnC)\). The complexity \cite{zhou2025spiking} of the interactions between \(Q\), \(K\) and \(V\) in original SSA is \(\mathcal{C}_{\text{SSA}}=\mathcal{O}(TN^2C)\). By introducing a differential mechanism, SCA reduces the computational cost of attention by a factor of \(N/n\):
\begin{equation}
    \frac{\mathcal{O}(TN^2C)}{\mathcal{O}(TNnC)}=\frac{N}{n} \gg 1.
\end{equation}
To demonstrate how this theoretical reduction in complexity translates to practical acceleration, we provide detailed evaluations of throughput and inference latency in Supp. \ref{appendix: throughput and inference latency}.
\section{Experiments}
\label{main: experiments}

\subsection{Image Classification}
\label{main: image classificatioon}

In the image classification task, we evaluate the efficacy of SCA on ImageNet-1K \cite{deng2009imagenet} using three representative state-of-the-art Spiking Transformer architectures: Spike-driven Transformer-V1 (SDT-V1) \cite{yao2023spike}, Spike-driven Transformer-V3 (SDT-V3) \cite{yao2025scaling}, and QKFormer \cite{zhou2024qkformer}, alongside a comprehensive comparative analysis against recent methods.

As shown in Table \ref{tab:image classification}, our SCA module consistently enhances performance across all three architectures. The most notable gains are observed in SDT-V1, achieving boosts of 1.03\% and 0.92\%, respectively. For SDT-V3, SCA yields improvements of 0.61\%, 0.63\%, and 0.86\% across backbone configurations of 5M, 10M, and 19M parameters. SCA introduces only negligible additional learnable parameters, which do not alter the reported model size after rounding. Similarly, QKFormer exhibits accuracy increases of 0.69\% and 0.73\% in its two configurations. These results demonstrate the versatility of the SCA module and validate its architecture-agnostic characteristic. Additionally, SCA enhances the foucus on key visual features (Fig. \ref{fig:attention map and fourier spectrum}(a) and \ref{fig:attention_heatmap_full}). Notably, the gains for QKFormer and SDT-V3 are more modest compared to SDT-V1. We attribute this to two factors: first, these models already possess strong baselines with limited room for saturation; second, their structural alignment with the Pyramid Vision Transformer \cite{liu2021swin,wang2021pyramid,yu2022metaformer} corroborates the finding that self-attention modules contribute relatively less to performance gains within the MetaFormer architecture \cite{yu2022metaformer}. Nevertheless, Fourier spectral analysis reveals an expansion of the high-frequency regions in SCA-enhanced models (Fig. \ref{fig:attention map and fourier spectrum}(b)), confirming the module's capability to restore high-frequency information. Our evaluation (Supp.\ref{appendix: source of gain}) also supports this: SCA yields substantial gains under high-pass filtering but only marginal gains under low-pass filtering, confirming its advantage lies in preserving high-frequency information rather than generic architectural enhancements.

\subsection{Semantic Segmentation and Event-based Tracking}
\label{main: semantic segmentation and event-based tracking tasks}

To further validate the efficacy of the proposed SCA, we extended our evaluation to more regression tasks, such as semantic segmentation and event-based tracking. For semantic segmentation, we employ the challenging ADE20K dataset \cite{zhou2017scene}, which comprises 20k and 2k images in the training and validation sets, covering 150 categories. We strictly adhere to the SDT-V3 experimental protocol \cite{yao2025scaling}. As shown in the Table \ref{tab:segmentation}, SCA achieved significant improvements of 2.68\%, 1.13\%, and 1.36\% in Mean Intersection over Union (MIoU) for model configurations with 5M, 10M, and 19M backbone parameters, respectively.

\begin{table}[htbp]
\centering
\vspace{-2mm}
\caption{Performance of segmentation.}
\vspace{-3mm}
\label{tab:segmentation}
\resizebox{\columnwidth}{!}{%
\begin{tabular}{lccc}
\toprule
\textbf{Architecture} & \textbf{Param (M)} & \textbf{Step} & \textbf{MIoU (\%)} \\
\midrule
ResNet-18 \cite{yu2022metaformer} & 15.5 & 1 & 32.9 \\
PVT-Tiny \cite{wang2021pyramid} & 17.0 & 1 & 35.7 \\
PVT-Small \cite{wang2021pyramid} & 28.2 & 1 & 39.8 \\
DeepLab-V3 \cite{zhang2022resnest} & 68.1 & 1 & 42.7 \\
\midrule
SDT-V2 \cite{yaospike} & 16.5 & 1 & 32.3 \\
SDT-V2 \cite{yaospike} & 16.5 & 4 & 33.6 \\
SDT-V2 \cite{yaospike} & 59.8 & 1 & 34.8 \\
SDT-V2 \cite{yaospike} & 59.8 & 4 & 35.3 \\
\cmidrule(l){1-4}
SDT-V3 \cite{yao2025scaling} & 5.1+1.4 & 4 & 33.6 \\
SDT-V3 \cite{yao2025scaling} & 10.0+1.4 & 4 & 40.1 \\
SDT-V3 \cite{yao2025scaling} & 19.0+1.4 & 4 & 41.3 \\
\cmidrule(l){1-4}
\textbf{SDT-V3-S+SCA (Ours)} & 5.1+1.4 & 4 & \textbf{36.28 \textcolor{violet}{($\uparrow$2.68)}} \\
\textbf{SDT-V3-M+SCA (Ours)} & 10.0+1.4 & 4 & \textbf{41.23 \textcolor{violet}{($\uparrow$1.13)}} \\
\textbf{SDT-V3-L+SCA (Ours)} & 19.0+1.4 & 4 & \textbf{42.76 \textcolor{violet}{($\uparrow$1.36)}} \\
\bottomrule
\end{tabular}%
}
\vspace{-2mm}
\end{table}

\begin{table*}[htbp]
\centering
\vspace{-2mm}
\caption{Performance on event-based tracking datasets.}
\vspace{-3mm}
\label{tab: tracking}
\resizebox{\textwidth}{!}{
\begin{tabular}{lcccccccc}
\toprule
\textbf{\multirow{2}{*}{Methods}} & \textbf{\multirow{2}{*}{\shortstack{Param.\\(M)}}} & \textbf{\multirow{2}{*}{\(T\)}} & \multicolumn{2}{c}{\textbf{FELT} \cite{wang2024long}} & \multicolumn{2}{c}{\textbf{FE108} \cite{zhang2021object}} & \multicolumn{2}{c}{\textbf{VisEvent} \cite{wang2023visevent}} \\
\cmidrule(lr){4-5} \cmidrule(lr){6-7} \cmidrule(lr){8-9}
& & & \textbf{AUC(\%)} & \textbf{PR(\%)} & \textbf{AUC(\%)} & \textbf{PR(\%)} & \textbf{AUC(\%)} & \textbf{PR(\%)} \\
\midrule
STARK \cite{yan2021learning} & 28.23 & 1 & 39.6 & 51.7 & 57.4 & 89.2 & 34.1 & 46.8 \\
ARTrack \cite{wei2023autoregressive} & 202.56 & 1 & 39.5 & 49.4 & 56.6 & 88.5 & 33.0 & 43.8 \\
OSTrack$_{256}$ \cite{ye2022joint} & 92.52 & 1 & 35.9 & 45.5 & 54.6 & 87.1 & 32.7 & 46.4 \\
HIPTrack \cite{cai2024hiptrack} & 120.41 & 1 & 38.2 & 48.9 & 50.8 & 81.0 & 32.1 & 45.2 \\
SNNTrack \cite{zhang2025spiking} & 31.40 & 5 & - & - & - & - & 35.4 & 50.4 \\
STNet \cite{zhang2022spiking} & 20.55 & 3 & - & - & - & - & 35.0 & 50.3 \\
SDTrack \cite{shan2025sdtrack} & 19.61 & 4 & 39.3 & 51.2 & 59.0 & 91.3 & 35.6 & 49.2 \\
\midrule
\textbf{SDTrack+SCA(Ours)} & 19.61 & 4 & \textbf{40.61}\textbf{\textcolor{violet}{($\uparrow$1.31)}} & \textbf{51.56}\textbf{\textcolor{violet}{($\uparrow$0.36)}} & \textbf{59.86}\textbf{\textcolor{violet}{($\uparrow$0.86)}} & \textbf{91.75}\textbf{\textcolor{violet}{($\uparrow$0.45)}} & \textbf{36.39}\textbf{\textcolor{violet}{($\uparrow$0.79)}} & \textbf{50.12}\textbf{\textcolor{violet}{($\uparrow$0.92)}} \\
\bottomrule
\end{tabular}}
\end{table*}

\begin{figure*}[htbp]
    \centering    
    \includegraphics[width=\linewidth]{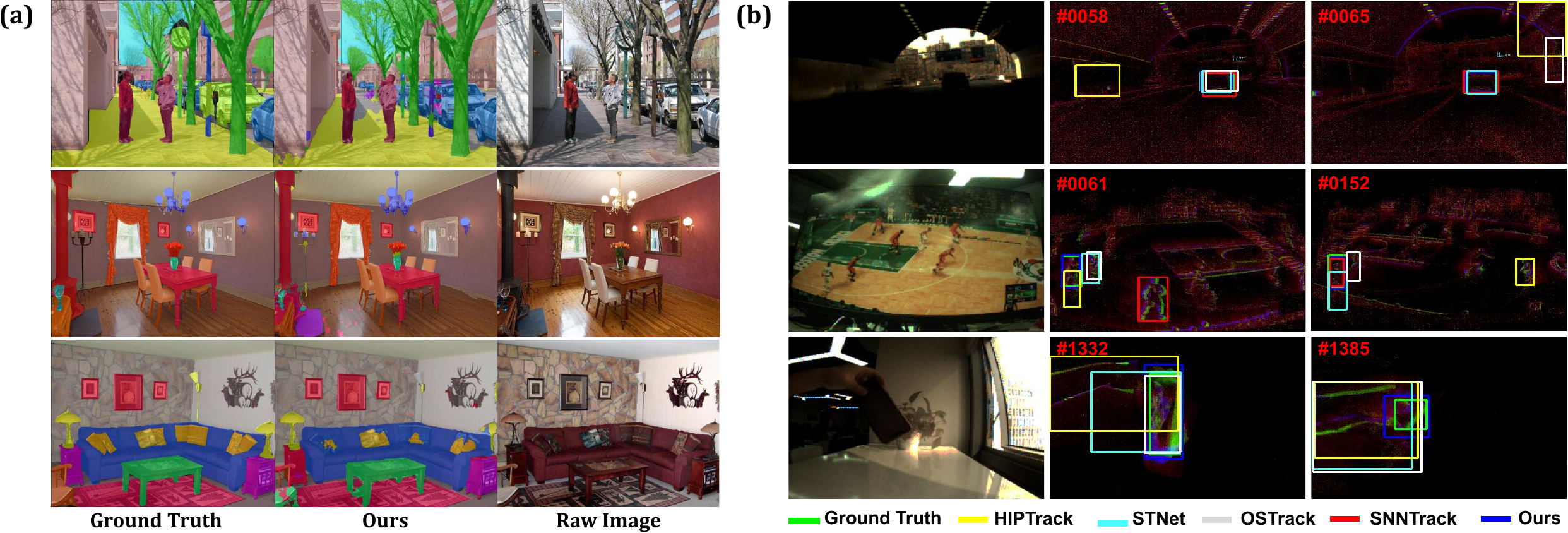}
    \vspace{-4mm}
    \caption{(a) Semantic segmentation results for ADE20K \cite{zhou2017scene} (SDT-V3-L+SCA). (b) Visualization of tracking results comparing our approach with other SOTA methods on the VisEvent \cite{wang2023visevent} dataset.}
    \vspace{-5mm}
    \label{fig:segmentation and tracking results}
\end{figure*}

Furthermore, we assessed SCA on event-based tracking, a challenging yet practical application domain for SNNs. We implemented the SDTrack Pipeline methodology \cite{shan2025sdtrack}, employing the Global Trajectory Prompt method to process event streams into event frames. Extensive experiments across the FE108 \cite{zhang2021object}, FELT \cite{wang2024long}, and VisEvent \cite{wang2023visevent} datasets consistently demonstrate that the SCA-enhanced SDTrack outperforms the original architecture on multiple metrics (as shown in table \ref{tab: tracking}). These results confirm the superior performance of SCA in complex regression tasks (provided in Fig. \ref{fig:segmentation and tracking results}(a), \ref{fig:segmentation_results_full} and Fig. \ref{fig:segmentation and tracking results}(b)).

\subsection{Ablation Study}

\textbf{Ablation on the Two-Stage Design.} We first investigate where original spiking self-attention (SSA) is replaced by our proposed method, including attention operations in Stage I (Global Contrast Aggregation) and Stage II (Local Differential Refinement). We employ the Spikformer architecture \cite{zhouspikformer}, comprising 4 blocks and 384 dimensions and conducted ablation studies on the CIFAR100 dataset \cite{krizhevsky2009learning}. The results in the Fig.\ref{fig:ablation_arch}(a) reveal a clear trend. When only Stage I is activated while Stage II uses SSA, performance decreases by 1.15\%. When only Stage II is activated while Stage I uses SSA, performance decreases by 1.07\%. This indicates that both global contrast aggregation and local differential refinement are crucial for visual tasks.

\textbf{Ablation on Contrast Prototype Construction.}  Fig.\ref{fig:ablation_arch}(b) investigates how to form two contrasting token streams for subsequent supply difference operations. Each token stream can be generated via pure learnable embedding (Emb.), via pooling (Pool), or via a combination (Pool+Emb.) that is adopted in our method. When only using embedding for both positive and negative streams, the model performance decreased by 0.52\% compared to the baseline. Replacing the positive stream with Pool resulted in a 0.17\% decrease in performance compared to the baseline. Replacing the negative stream with Pool led to a 0.43\% decrease in performance compared to the baseline. Performance peaked when both positive and negative streams employed Pool and embedding. These results indicate that combining the two operations achieves the strongest synergistic effect.

\textbf{Component-wise Ablation of SCA.} We have conducted a granular ablation study (Table \ref{tab:fine-grained ablation}) on CIFAR-100 using Spikformer (Origin SSA). The results demonstrate that (a) Contrast-token compression and Learnable embedding do not provide accuracy gains. Their primary roles are strictly to reduce computational complexity to \(\mathcal{O}(N)\) and stabilize token routing, respectively; (b) Dual-stream separation provides a modest architectural inductive bias; Differential fusion is the undeniable core driver of our performance leap. These results show that the improvements do not arise simply from an effective low-rank or global-local architectural design (which alone slightly degrades accuracy). Instead, the gains are driven by the differential fusion, confirming our central claim that explicitly mitigate the spectral bias is the primary source of improvement (as detailed in Supp. \ref{appendix: source of gain}).

\begin{figure}[htbp]
    \centering
    \vspace{-3mm}
    % --- 第一个元素：图 1 ---
    % 将 [b] 改为 [c]，实现垂直居中对齐
    \begin{minipage}[c]{0.49\textwidth}
        \centering
        \includegraphics[width=\linewidth]{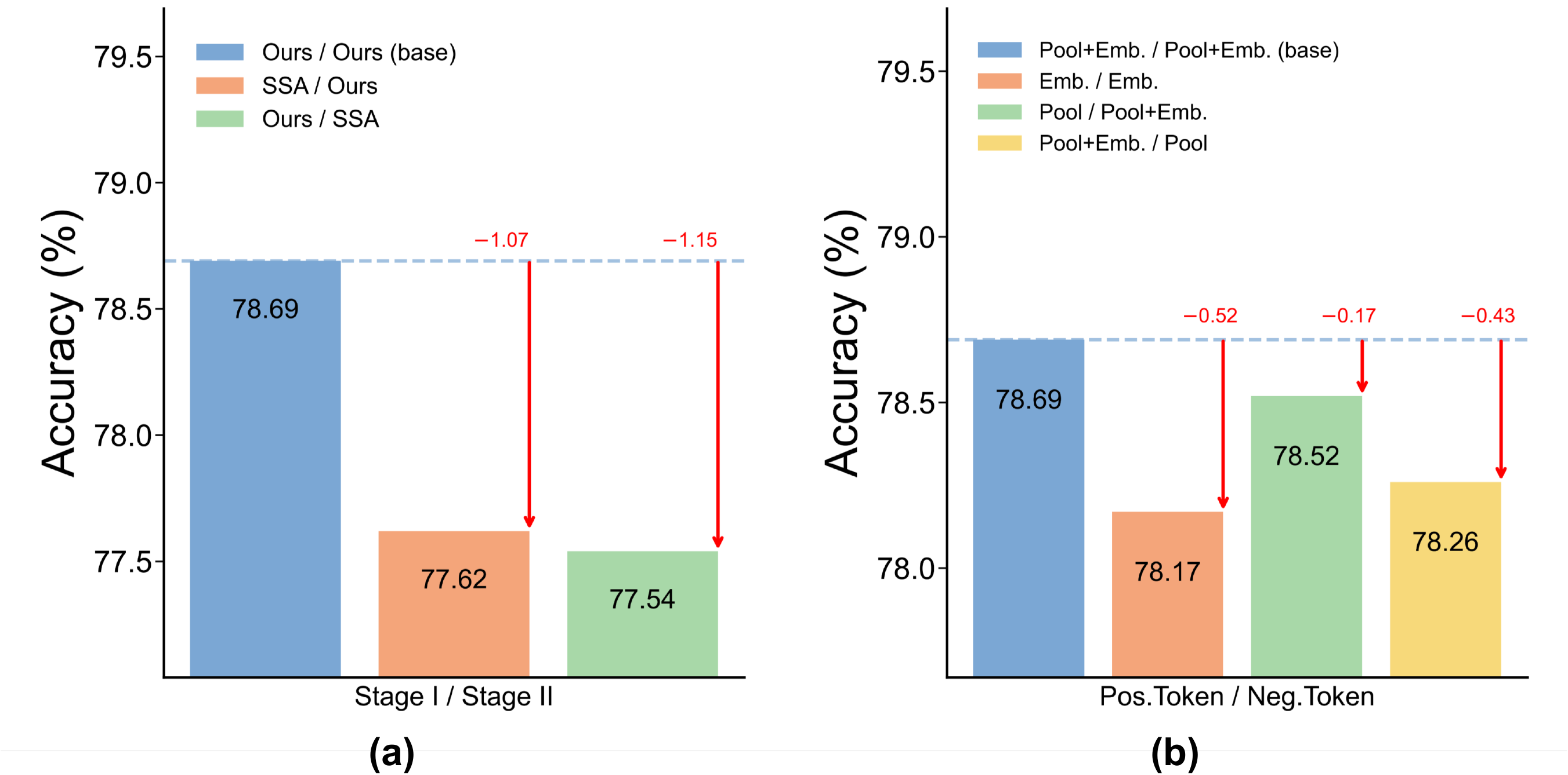}
        \vspace{-5mm}
        \caption{Ablation study of (a) architecture and (b) token generation on Spikformer.}
        \label{fig:ablation_arch}
    \end{minipage}
    \hfill % 自动撑开中间的空白
    % --- 第二个元素：表格 ---
    % 将 [b] 改为 [c]，实现垂直居中对齐
    \begin{minipage}[c]{0.489\textwidth}
        \centering
        \captionof{table}{Ablation study of components in SCA.}
        \label{tab:fine-grained ablation}
        \resizebox{\linewidth}{!}{
        \begin{tabular}{lc}
            \toprule
            \textbf{Method} & \textbf{Acc.(\%)} \\
            \midrule
            Origin SSA & 77.86 \\
            + Contrast-token compression & 77.74 \textcolor{violet}{($\downarrow$0.12)} \\
            + Dual-stream separation & 78.04 \textcolor{violet}{($\uparrow$0.18)} \\
            + Learnable embedding & 77.65 \textcolor{violet}{($\downarrow$0.21)}\\
            + Differential fusion & 78.69 \textcolor{violet}{($\uparrow$0.83)}\\
            \textbf{SCA (Ours)} & \textbf{78.69} \textcolor{violet}{($\uparrow$0.83)}\\
            \bottomrule
        \end{tabular}}
    \end{minipage}
    \vspace{-4mm}
\end{figure}

\section{Conclusion}

In this work, we delve into the performance gap between Spiking Transformers and ANN counterpart. Through theoretical analysis, we reveal spectral bias as the key factor limiting models' performance. To address this issue, we propose the Spiking Contrastive Attention (SCA) which inspired by edge detection and biological vision. SCA efficiently aggregates global context at lower complexity by mimicking the differential perception characteristics of biological neurons. Extensive experiments across image classification, semantic segmentation, and event-based tracking demonstrate that SCA improves model performance while mitigating the attenuation of high-frequency information. We believe SCA not only offers a novel perspective for addressing feature degradation in SNNs but also holds potential to become a foundational component for constructing next-generation energy-efficient Spiking Transformers.
{
    \small
    \bibliographystyle{ieeenat_fullname}
    \bibliography{main}
}

\clearpage
\renewcommand{\thesection}{\Alph{section}}
\setcounter{section}{0}
\setcounter{page}{1}
\maketitlesupplementary

\section{Supplementary Material}
\label{appendix: preliminry}

\subsection{Discrete Time Fourier Transform}
\label{appendix: discrete time fourier transform}

The Discrete-Time Fourier Transform (DTFT) serves as a fundamental mathematical tool for analyzing the properties of non-periodic discrete-time signals. It maps discrete sequences from the time domain to continuous functions in the frequency domain, providing essential theoretical underpinning for understanding spectral bias and frequency-domain filtering characteristics in deep neural networks. 

\subsubsection{Definition}
\label{appendix: definition of DTFT}

For a discrete-time signal \(x[n]\), where \(n \in \mathbb{Z}\), its Discrete-Time Fourier Transform (DTFT) \(X(e^{j\omega})\) is defined as:
\begin{equation}
    X(e^{j\omega}) = \sum_{n=-\infty}^{\infty} x[n] e^{-j\omega n},
\end{equation}
where \(\omega\) represents the normalized angular frequency in \(\text{radians/sample}\). Typically, \(X(e^{j\omega})\) is a complex-valued function comprising a magnitude spectrum \(|X(e^{j\omega})|\) and a phase spectrum \(\angle X(e^{j\omega})\).

\subsubsection{Convergence Conditions}
\label{appendix: convergence conditions}

The existence of the DTFT is not guaranteed for all discrete sequences. Typically, the sequence \(x[n]\) is required to satisfy the condition of \textbf{Absolute Summability}:
\begin{equation}
    \sum_{n=-\infty}^{\infty} |x[n]| < \infty.
\end{equation}
If this condition is met, the series converges uniformly to \(X(e^{j\omega})\). For finite-energy signals (which are square-summable, \textit{i.e.}, \(\sum |x[n]|^2 < \infty\)), the DTFT converges in the mean-square sense. In this context, the spectrum typically precludes the presence of Dirac \(\delta\) functions.

\subsubsection{Convolution Theorem}
\label{appendix: convolution theorem}

The convolution of two signals \(x[n]\) and \(h[n]\) in the time domain corresponds to the product of their DTFTs in the frequency domain:
\begin{equation}
\begin{split}
    y[n] &= x[n] * h[n] = \sum_{k=-\infty}^{\infty} x[k] h[n-k] \\
    &\quad \xleftrightarrow{\text{DTFT}} \quad Y(e^{j\omega}) = X(e^{j\omega}) H(e^{j\omega}).
\end{split}
\end{equation}
This property demonstrates that the convolution operation fundamentally acts as \textbf{frequency weighting} on the input signal \(X(e^{j\omega})\), with weights determined by the frequency response \(H(e^{j\omega})\) of the convolution kernel. This offers a direct mathematical perspective for interpreting feature extraction and filtering mechanisms within neural networks.

\subsection{Schur Product Theorem}
\label{appendix: schur product theorem}

This section details the Schur Product Theorem and its application within the softmax attention mechanism. Furthermore, we elucidate Equ. \ref{equ: schur's application} from the main text, specifically clarifying how the element-wise exponential operation restores the full-rank property to \(\exp(\mathbf{S})\), despite the input \(\mathbf{S}\) being a low-rank matrix.

\subsubsection{Definition}
\label{appendix: definition of schur product theorem}

Given two matrices \(\mathbf{A}, \mathbf{B} \in \mathbb{R}^{N \times N}\) of identical dimensions, their Schur Product is defined as the element-wise product, denoted by \(\mathbf{A} \circ \mathbf{B}\):
\begin{equation}
    [\mathbf{A} \circ \mathbf{B}]_{ij} = [\mathbf{A}]_{ij} [\mathbf{B}]_{ij}, \quad 1 \leq i, j \leq N.
\end{equation}
The classical Schur Product Theorem \cite{schur1911bemerkungen} primarily establishes the preservation of positive definiteness under this operation.

\begin{theorem}
\label{thm: schur product theorem}
If \(\mathbf{A}\) and \(\mathbf{B}\) are positive semi-definite (PSD) matrices, \textit{i.e.}, \(\mathbf{A} \succeq 0\) and \(\mathbf{B} \succeq 0\), then their Schur Product is also positive semi-definite: \(\mathbf{A} \circ \mathbf{B} \succeq 0\).
\end{theorem}

This property guarantees that the element-wise product of two kernel matrices remains a valid kernel matrix.

\subsubsection{Connection to Element-wise Exponential}
\label{appendix: connection to element-wise exponential}

The core operation of softmax is the element-wise exponential \(\exp(\mathbf{S})\). For a matrix \(\mathbf{S}\), this can be expanded via a Maclaurin series into an infinite weighted sum of Schur powers:
\begin{equation}
\begin{split}
    \exp(\mathbf{S}) &= \mathbf{1}\mathbf{1}^\top + \mathbf{S} + \frac{1}{2!} (\mathbf{S} \circ \mathbf{S}) \\
    &\quad + \frac{1}{3!} (\mathbf{S} \circ \mathbf{S} \circ \mathbf{S}) + \dots \\
    &= \sum_{k=0}^{\infty} \frac{1}{k!} \mathbf{S}^{\circ k},
\end{split}
\end{equation}
where \(\mathbf{S}^{\circ k}\) denotes the \(k\)-th Schur product of \(\mathbf{S}\), and \(\mathbf{S}^{\circ 0}\) represents the all-ones matrix.

According to Theorem \ref{thm: schur product theorem}, if \(\mathbf{S}\) is positive semi-definite, then \(\exp(\mathbf{S})\), being a nonnegative linear combination of positive semi-definite matrices, must also be positive semi-definite (and is typically positive definite).

\subsubsection{Rank Restoration via Schur Product}
\label{appendix: rank restoration via schur product}

In vanilla self-attention, \(\mathbf{S} = \mathbf{Q}\mathbf{K}^\top\). Given that \(\mathbf{Q}, \mathbf{K} \in \mathbb{R}^{N \times d_k}\) and typically \(d_k \ll N\), the matrix \(\mathbf{S}\) is low-rank, \textit{i.e.}, \(\text{rank}(\mathbf{S}) \leq d_k\).

However, the Schur Product possesses the property of increasing matrix rank. For two matrices \(\mathbf{A}\) and \(\mathbf{B}\) with ranks \(r_a\) and \(r_b\) respectively, the rank of their Schur Product satisfies the following inequality:
\begin{equation}
    \text{rank}(\mathbf{A} \circ \mathbf{B}) \leq \text{rank}(\mathbf{A}) \text{rank}(\mathbf{B}).
\end{equation}
While this represents an upper bound, in the generic case, the rank of \(\mathbf{S} \circ \mathbf{S}\) tends to approach \(\min(N, d_k^2)\), which typically substantially exceeds the original rank \(d_k\) of \(\mathbf{S}\).

When examining the expansion \(\exp(\mathbf{S}) = \sum_{k=0}^{\infty} \frac{1}{k!} \mathbf{S}^{\circ k}\): a. This sum includes all higher-order Schur powers of \(\mathbf{S}\); b. Each term \(\mathbf{S}^{\circ k}\) spans a higher-dimensional eigenspace; c. Unless \(\mathbf{S}\) possesses a highly specific sparse structure (such as a diagonal matrix), the linear combinations of these Schur powers of different orders are extremely likely to be linearly independent, thereby filling the entire \(N \times N\) space.

Consequently, \(\exp(\mathbf{S})\) is full rank almost everywhere: \(\text{rank}(\exp(\mathbf{S})) = N\). Given \(\mathbf{P}_{\text{soft}} = \mathbf{D}^{-1} \exp(\mathbf{S})\), where \(\mathbf{D}^{-1}\) is an invertible diagonal matrix (and thus full rank), and since multiplication by a non-singular matrix preserves rank, it follows that:
\begin{equation}
    \text{rank}(\mathbf{P}_{\text{soft}}) = \min(\text{rank}(\mathbf{D}^{-1}), \text{rank}(\exp(\mathbf{S}))) = N.
\end{equation}
This substantiates the conclusion presented in the main text: the softmax operation restores the rank from \(d_k\) to \(N\), rendering the kernel space of the attention matrix trivial, \textit{i.e.}, \(\mathcal{N}(\mathbf{P}_{\text{soft}}) = \{\mathbf{0}\}\).

\subsection{Vanilla Self-Attention}
\label{appendix: vanilla self attention}

Vanilla self-attention (VSA) is the core component making ViT successful \cite{vaswani2017attention}. By weighting the image-patches feature value through the dot-product of query and key and softmax function, self-attention can capture the global dependence and interest representation \cite{katharopoulos2020transformers,qin2022cosformer}. Given an input feature sequence \(X \in \mathbb{R}^{T \times N \times D}\), the VSA has three float-point key components, namely query (\(Q_F\)), key (\(K_F\)), and value (\(V_F\)) which are calculated by learnable linear matrices \(W_Q,W_K,W_V \in \mathbb{R}^{D \times D}\) and \(X\):
\begin{equation}
    Q_{\mathcal{F}} = X W_Q, \, K_{\mathcal{F}} = X W_K, \, V_{\mathcal{F}} = X W_V, 
\end{equation}
where \(\mathcal{F}\) denotes the float-point form. The output of vanilla self-attention can be computed as:
\begin{equation}
    \text{VSA}(Q_{\mathcal{F}}, K_{\mathcal{F}}, V_{\mathcal{F}}) = \text{softmax} \left( \frac{Q_{\mathcal{F}} K_{\mathcal{F}}^{\mathrm{T}}}{\sqrt{d}} \right) V_{\mathcal{F}},
\end{equation}
where \(d=D/H\) is the feature dimension of one head and \(H\) is the head number. 

\subsection{Softmax-free Spiking Self-Attention}
\label{appendix: softmax-free spiking self attention}

For the following two reasons, VSA computation is not applicable for SNNs: a. Floating-point matrix multiplication involving \(Q_F\), \(K_F\), and the softmax operation includes exponential computations and division operations, which violate SNN computational rules; b. The quadratic space and time complexity of VSA sequence length fails to meet the efficient computation requirements of SNNs. Consequently, numerous studies have focused on developing attention mechanisms suitable for SNNs. This section provides a detailed introduction to the three spiking self-attention mechanisms analyzed in the main text that utilize LIF neurons as activation functions.

\subsubsection{Spiking Self-Attention}
\label{appendix: spiking self attention}

Zhou et al. \cite{zhouspikformer} proposed Spikformer, the first framework integrating the strengths of SNNs and Transformers. The Spiking Self Attention (SSA) mechanism in Spikformer simulates sparse visual features by using spiking queries, keys, and values without softmax. Its computation proceeds as follows:
\begin{align}
\label{equ: SSA}
&\begin{aligned}
 Q&=\mathcal{SN}_Q(\text{BN}(XW_Q)),\\
 K&=\mathcal{SN}_K(\text{BN},(XW_K)),\\
 V&=\mathcal{SN}_V(\text{BN}(XW_V))
\end{aligned} \\
&\text{SSA}'(Q,K,V)=\mathcal{SN}(QK^TV*s), \\
&\text{SSA}(Q,K,V)=\mathcal{SN}(\text{BN}(\text{Linear}(\text{SSA}'(Q,K,V)))),
\end{align}
where \(Q,K,V\ in \mathbb{R}^{T \times N \times D}\), \(s\) is a scaling factor which control the large value of matrix multiplication result.

\subsubsection{Spike-driven Self-Attention}
\label{appendix: spike-driven self attention}

Yao et al. \cite{yao2023spike} proposed Spike-driven Self Attention (SDSA) using only masks and addition operations. Given a spike input feature sequence \(S \in \mathbb{R}^{T \times N \times D}\), float-point \(Q\), \(K\), and \(V\) in \(\mathbb{R}^{T \times N \times D}\) are calculated by three learnable linear matrices, respectively. A spike neuron layer \(\mathcal{SN}(\cdot)\) follows, converting \(Q,K,V\) into spike tensor \(Q_S,K_S\) and \(V_S\). SDSA is presented as:
\begin{equation}
\label{equ: SDSA}
\begin{split}
    \text{SDSA}(Q,K,V)&=g(Q_S,K_S) \otimes V_S\\
    &=\mathcal{SN}(\text{SUM}_c(Q_S \otimes K_S)) \otimes V_S,
\end{split}
\end{equation}
where \(\otimes\) is the Hadamard product, \(g(\cdot)\) is used to compute the attention map, \(\text{SUM}_c\) represents the sum of each column. The outputs of both \(g(\cdot)\) and \(\text{SUM}_c(\cdot)\) are \(D\)-dimensional row vectors. The Hadamard product between spike tensors is equivalent to the mask operation.

\subsubsection{Meta Spike-driven Self-Attention}
\label{appendix: meta spike-driven self attention}

Meta-SpikeFormer \cite{yaospike} explores the impact of architecture, spike-driven self-attention, and skip connections on performance by extending Yao et al.'s Spike-driven Transformer \cite{yao2023spike} into a meta architecture. To distinguish it from SDSA, the attention mechanism in Meta-SpikeFormer will henceforth be referred to as Meta Spike-driven Self-Attention (Meta-SDSA). This paper employs the default attention mechanism in Meta-SpikeFormer (\textit{i.e.}, SDSA-3 in \cite{yaospike}) for analysis. Meta-SDSA use matrix multiplication, which is written as:
\begin{equation}
\label{equ: Meta-SDSA}
\begin{split}
    \text{Meta-SDSA}(Q_S,K_S,V_S)&=\mathcal{SN}_s(Q_S(K_S^T V_S))\\
    &=\mathcal{SN}_s((Q_SK_S^T)V_S),
\end{split}
\end{equation}
where \(\mathcal{SN}_s\) is the Spike layer with the threshold \(s \cdot u_{th}\). Meta-SDSA directly merge the \(s\) into the threshold of the spiking neuron to circumvent the multiplication by \(s\).

\subsection{Main Theorems and Proofs}
\label{appendix: main theorems and proofs}

\subsubsection{Proof of \texorpdfstring{\cref{thm: fourier analysis of spiking neuron}}{Theorem 3.1}}
\label{appendix: proof of theorem 3.1}

% \begin{customthm}{3.1}
% Applying the Discrete-Time Fourier Transform (DTFT) to the dynamic equation yields the frequency response function of the LIF neuron:
% \begin{equation}
%     H(\omega)=\frac{V(\omega)}{X(\omega)}=\frac{K}{1- \alpha e^{-j \omega}},
% \end{equation}
% where \(\omega\) denotes the angular frequency. Analysis of \(H(\omega)\) reveals that low-frequency signals pass with negligible attenuation, whereas in the high-frequency regime, the gain is minimal, resulting in severe signal attenuation. Consequently, the LIF neuron functions as a discrete-time low-pass filter in the frequency domain.
% \end{customthm}

\begin{customthm}{3.1}
The charging process of a single LIF neuron can be equivalently modeled as a first-order Infinite Impulse Response (IIR) low-pass filter, where its squared amplitude-frequency response decreases strictly monotonically with the digital angular frequency \(\omega \in [0, \pi]\). For an \(L\)-layer SNN, assuming the synaptic weight frequency response and the equivalent spike gain of the \(i\)-th layer are denoted by \(W_i(e^{j\omega})\) and \(k_i\) respectively, the overall end-to-end frequency response \(H_L(e^{j\omega})\) is exponentially modulated by the low-pass factor, specifically formulated as (proof in Supp. \ref{appendix: proof of theorem 3.1}):
\begin{equation}
    H_L(e^{j\omega}) = \left(\prod_{i=1}^L k_i W_i(e^{j\omega})\right) \left(\frac{1 - \beta}{1 - \beta e^{-j\omega}}\right)^L.
\end{equation}
\end{customthm}

\begin{proof}

For analytical tractability, we first isolate the sub-threshold charging dynamics of the LIF neuron, while the spike-triggered reset is treated as a nonlinear event modulation in the subsequent network-level analysis.

\textbf{DTFT Analysis of the Charging Process.} First, we isolate the charging equation of the LIF neuron's membrane potential in the absence of spikes. Assuming the neuron is in a sub-threshold state (\(S[t] = 0\)), we set the resting potential $V_{reset} = 0$ without loss of generality to facilitate linear frequency-domain analysis.

Based on Equ. \ref{eq: lif3}, we have $V[t] = H[t]$. Substituting this into Equ. \ref{eq: lif1} yields:
\begin{align}
    & V[t] = V[t-1] + \frac{1}{\tau}(X[t] - V[t-1]), \\
    & V[t] = \left(1 - \frac{1}{\tau}\right)V[t-1] + \frac{1}{\tau}X[t].
\end{align}

Let \(\beta = 1 - \frac{1}{\tau}\) denote the membrane potential decay factor. For a LIF neuron, the time constant \(\tau > 1\), therefore \(0 < \beta < 1\). The difference equation simplifies to:
\begin{equation}
    V[t] = \beta V[t-1] + (1 - \beta)X[t].
\end{equation}

Applying the Discrete-Time Fourier Transform (DTFT) to both sides and utilizing the time-shifting property \(x[t-1] \xrightarrow{DTFT} e^{-j\omega}X(e^{j\omega})\), we obtain:
\begin{equation}
    V(e^{j\omega}) = \beta e^{-j\omega}V(e^{j\omega}) + (1 - \beta)X(e^{j\omega}).
\end{equation}

Rearranging the terms yields the frequency response \(H_{charge}(e^{j\omega})\) from the input \(X(e^{j\omega})\) to the membrane potential \(V(e^{j\omega})\):
\begin{align}
    & V(e^{j\omega})(1 - \beta e^{-j\omega}) = (1 - \beta)X(e^{j\omega}), \\
    & H_{charge}(e^{j\omega}) = \frac{1 - \beta}{1 - \beta e^{-j\omega}}.
\end{align}

Mathematically, this represents the transfer function of a standard first-order Infinite Impulse Response (IIR) filter.

\textbf{Proof of Amplitude-Frequency Characteristics.} To rigorously prove its low-pass filtering properties, we compute the squared amplitude-frequency response \(|H_{charge}(e^{j\omega})|^2\):
\begin{align}
& |H_{charge}(e^{j\omega})|^2 = \frac{(1-\beta)^2}{|1 - \beta(\cos\omega - j\sin\omega)|^2}, \\
& |H_{charge}(e^{j\omega})|^2 = \frac{(1-\beta)^2}{(1 - \beta\cos\omega)^2 + (\beta\sin\omega)^2}, \\
& |H_{charge}(e^{j\omega})|^2 = \frac{(1-\beta)^2}{1 - 2\beta\cos\omega + \beta^2\cos^2\omega + \beta^2\sin^2\omega}, \\
& |H_{charge}(e^{j\omega})|^2 = \frac{(1-\beta)^2}{1 - 2\beta\cos\omega + \beta^2}.
\end{align}

Next, we evaluate the magnitude at the frequency extremes (\(\omega = 0\) and \(\omega = \pi\)): \textit{a. At the lowest frequency} (\(\omega = 0\)), \(\cos(0) = 1\):
\begin{equation}
    |H_{charge}(e^{j0})|^2 = \frac{(1-\beta)^2}{1 - 2\beta + \beta^2} = \frac{(1-\beta)^2}{(1-\beta)^2} = 1.
\end{equation}

\textit{b. At the highest frequency} (\(\omega = \pi\)), \(\cos(\pi) = -1\)):
\begin{equation}
    |H_{charge}(e^{j\pi})|^2 = \frac{(1-\beta)^2}{1 + 2\beta + \beta^2} = \frac{(1-\beta)^2}{(1+\beta)^2}.
\end{equation}

Because \(0 < \beta < 1\), it strictly follows that \((1-\beta)^2 < (1+\beta)^2\), yielding \(|H_{charge}(e^{j\pi})|^2 < 1\). 

In the interval \(0 \le \omega \le \pi\), as the frequency \(\omega\) increases, \(\cos\omega\) strictly decreases. This causes the denominator \(1 - 2\beta\cos\omega + \beta^2\) to monotonically increase, resulting in a strictly monotonic decrease in the overall amplitude response. This rigorously proves that the charging process of the LIF neuron inherently functions as a low-pass filter.

\textbf{Amplification of Low-Pass Characteristics at the Network Level.} While the spike emission process is highly non-linear, we can linearize it around the firing threshold \(V_{th}\) by approximating the step function \(\Theta\) with a local gain \(k\). In the frequency domain, the DTFT of the output spike \(S(e^{j\omega})\) is approximated as:
\begin{equation}
    S(e^{j\omega}) \approx k V(e^{j\omega}).
\end{equation}

When this output spike propagates through the subsequent synaptic weight network (with frequency response \(W(e^{j\omega})\)), the total transfer function of a single-layer neuron—from input to weighted output—becomes:
\begin{equation}
    H_{sys}(e^{j\omega}) = k W(e^{j\omega}) \frac{1 - \beta}{1 - \beta e^{-j\omega}}.
\end{equation}

In a SNN cascading across \(L\) layers, the overall end-to-end system frequency response evaluates to:
\begin{equation}
    H_L(e^{j\omega}) = \left(\prod_{i=1}^L k_i W_i(e^{j\omega})\right) \left(\frac{1 - \beta}{1 - \beta e^{-j\omega}}\right)^L.
\end{equation}

Here, the low-pass factor \(\left(\frac{1 - \beta}{1 - \beta e^{-j\omega}}\right)^L\) undergoes exponential amplification. This demonstrates that, even when accounting for synaptic weights and spike emission, the macroscopic SNN architecture strongly suppresses high-frequency noise and facilitates the transmission of low-frequency signals.

\textbf{When considering reset mechanism.} We extended Theorem 3.1 to mathematically incorporate the reset mechanism, demonstrating that the low-pass characteristic not only holds but is compounded by reset.

Let \(S[t-1]\) represent spike emission at time \(t-1\). We rewrite the full neural dynamic:
\begin{equation}
    V[t]=\alpha V[t-1](1-S[t-1])+KX[t].
\end{equation}

To isolate the nonlinear reset effect, we define "reset penalty" \(P[t] = \alpha V[t-1]S[t-1]\), representing membrane potential discarded upon crossing threshold. The equation simplifies to:
\begin{equation}
    V[t]=\alpha V[t-1]+KX[t]-P[t].
\end{equation}

Applying DTFT:
\begin{equation}
    V(\omega) = \alpha e^{-j\omega} V(\omega) + K X(\omega) - P(\omega).
\end{equation}

Solving for effective frequency response \(H_{eff}(\omega)=\frac{V(\omega)}{X(\omega)}\): 
\begin{equation}
    H_{eff}(\omega) = \frac{K}{1 - \alpha e^{-j\omega}} - \frac{P(\omega)}{X(\omega)(1 - \alpha e^{-j\omega})}.
\end{equation}

The first term is identical to low-pass filter \(H(\omega)\) from \cref{thm: fourier analysis of spiking neuron}, proving the RC-circuit leak mechanism governs baseline spectral property. The second term represents attenuation from reset. When sharp high-frequency features cause rapid \(V[t]\) accumulation, triggering a spike (\(S[t-1]=1\)), \(P[t]\) becomes non-zero, actively subtracting energy from the spectrum.

Therefore, incorporating reset reveals it acts as a nonlinear amplitude truncator (negative feedback). It does not alter the low-pass trend but further restricts the neuron's ability to transmit sustained high-frequency fluctuations. We will add this extended proof to the text to make theoretical foundation more robust.

\end{proof}

\subsubsection{Proof of Lemma \ref{lem: soft-free}}
\label{appendix: proof of lemma 3.3}

\begin{customlemma}{3.3}
If the feature dimension \(d_k < N\), the Dirichlet Energy of the output signal \(\mathbf{y}_{\text{soft-free}} = \mathbf{P}_{\text{soft-free}} \mathbf{x}\) is constrained by the null space of \(\mathbf{P}_{\text{soft-free}}\), resulting in the truncation of high-frequency components.
\end{customlemma}

\begin{proof}
\textbf{a. Rank Bounds.} Given \(\mathbf{P}_{\text{soft-free}} \propto \mathbf{Q}\mathbf{K}^T\) with \(\mathbf{Q}, \mathbf{K} \in \mathbb{R}^{N \times d_k}\), matrix rank inequalities yield:
\begin{equation}
    \text{rank}(\mathbf{P}_{\text{soft-free}}) \leq \min(\text{rank}(\mathbf{Q}), \text{rank}(\mathbf{K})) \leq d_k.
\end{equation}

Since typically \(N \gg d_k\), \(\mathbf{P}_{\text{soft-free}}\) is inherently a low-rank matrix.

\textbf{b. Orthogonal Decomposition and Energy Loss.} We decompose the input signal \(\mathbf{x}\) into \(\mathbf{x} = \mathbf{x}_{\parallel} + \mathbf{x}_{\perp}\), where \(\mathbf{x}_{\parallel} \in \text{RowSpace}(\mathbf{P}_{\text{soft-free}})\) and \(\mathbf{x}_{\perp} \in \mathcal{N}(\mathbf{P}_{\text{soft-free}})\) (the null space). As \(\text{dim}(\mathcal{N}(\mathbf{P}_{\text{soft-free}})) \ge N - d_k\), the vast majority of degrees of freedom reside within the null space. In natural signals, high-frequency noise and local details are typically orthogonal to principal components, implying that \(\mathbf{x}_{\perp}\) contains substantial high-frequency energy. The output signal is given by:
\begin{equation}
    \mathbf{y}_{\text{soft-free}} = \mathbf{P}_{\text{soft-free}}(\mathbf{x}_{\parallel} + \mathbf{x}_{\perp}) = \mathbf{P}_{\text{soft-free}}\mathbf{x}_{\parallel}.
\end{equation}

Consequently, all energy within \(\mathbf{x}_{\perp}\), including high-frequency variations, is nullified. The operator thus functions as a low-pass filter, resulting in:
\begin{equation}
    E_{\text{Dirichlet}}(\mathbf{y}_{\text{soft-free}}) \ll E_{\text{Dirichlet}}(\mathbf{x}).
\end{equation}

Thus, high-frequency components are truncated.
\end{proof}

\subsubsection{Proof of \texorpdfstring{\cref{thm: soft}}{Theorem 3.4}}
\label{appendix: proof of theorem 3.4}

\begin{customthm}{3.4}
The softmax operation employs an element-wise exponential transformation to map the rank-\(d_k\) linear matrix to a full-rank matrix, thereby maximizing the preservation of Dirichlet Energy.
\end{customthm}

\begin{figure*}[t]
    \centering
    \includegraphics[width=\linewidth]{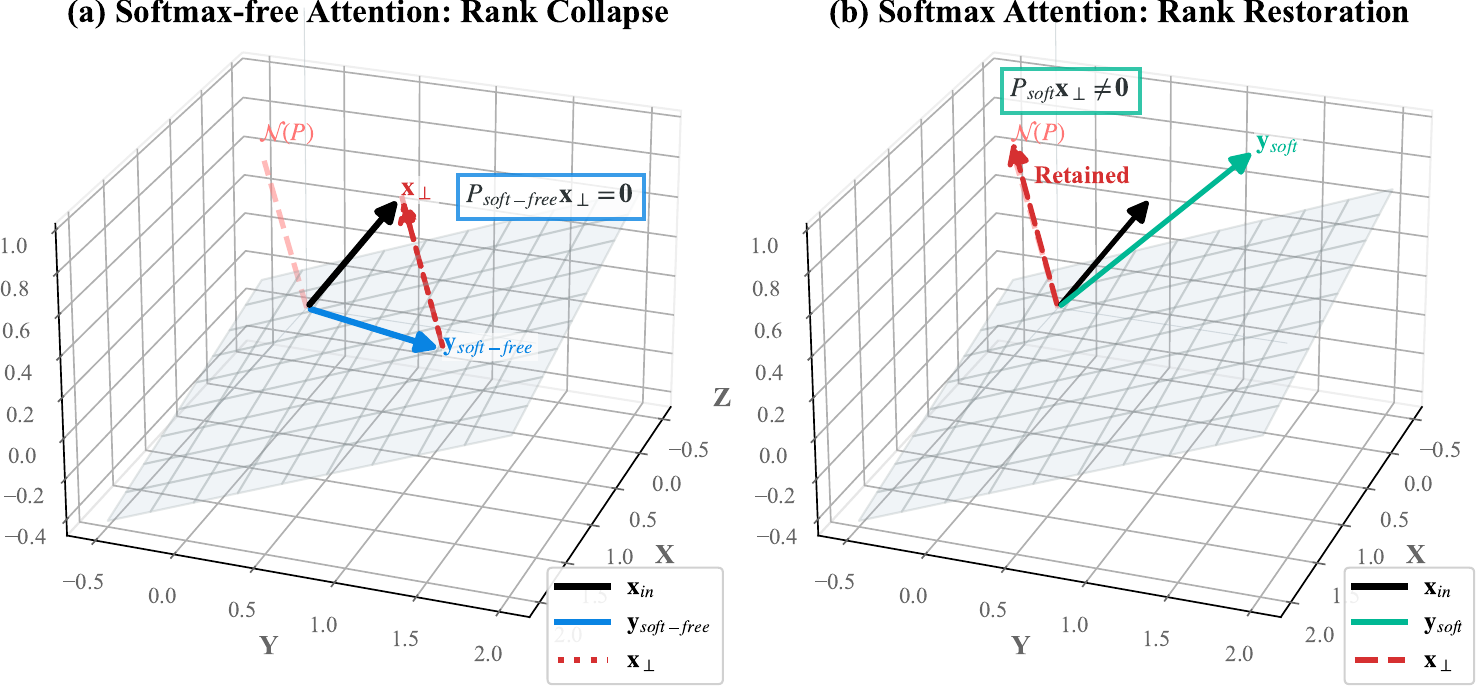}
    \caption{Rank restoration. Refer to \ref{appendix: proof of theorem 3.4}}
    \label{fig:rank restoration}
\end{figure*}

\begin{proof}
\label{proof: soft}
\textbf{a. Rank Recovery.} Let \(\mathbf{S}=\mathbf{QK}^T\). Given \(\mathbf{Q}, \mathbf{K} \in \mathbb{R}^{N \times d_k}\) where \(d_k \ll N\), by matrix rank inequalities:
\begin{equation}
    \text{rank}(\mathbf{S}) \leq \min(\text{rank}(\mathbf{Q}), \text{rank}(\mathbf{K})) \leq d_k,
\end{equation}
This implies that \(\mathbf{S}\) is a low-rank matrix, whose column vectors span only a low-dimensional subspace of \(\mathbb{R}^N\).

The core of softmax lies in the element-wise exponential operation \(\mathbf{A}=\exp(\mathbf{S})\). Based on the Schur Product Theorem (refer to \ref{appendix: schur product theorem}) \cite{schur1911bemerkungen,fuzhen2005schur} and the properties of element-wise exponentials, for a low-rank matrix \(\mathbf{S}\), its element-wise exponential matrix \(\exp(\mathbf{S})\) is full-rank (refer to \ref{appendix: rank restoration via schur product}). That is:
\begin{equation}
\label{equ: schur's application}
    \text{rank}(\mathbf{P}_{\text{soft}}) = \text{rank}(\mathbf{D}^{-1}\exp(\mathbf{S})) = N,
\end{equation}
where \(\mathbf{D}\) denotes the normalization diagonal matrix. Consequently, the softmax operation restores the rank of the attention matrix from \(d_k\) to \(N\). This implies that the kernel of \(\mathbf{P}_{\text{soft}}\) is trivial, \textit{i.e.}, \(\mathcal{N}(\mathbf{P}_{\text{soft}}) = \{\mathbf{0}\}\).

\textbf{b. Orthogonal Decomposition and Energy Loss}. The input graph signal \(\mathbf{x}\) can be orthogonally decomposed as \(\mathbf{x} = \mathbf{x}_{\parallel} + \mathbf{x}_{\perp}\), where \(\mathbf{x}_{\parallel} \in \text{RowSpace}(\mathbf{P}_{\text{soft-free}})\) and \(\mathbf{x}_{\perp} \in \mathcal{N}(\mathbf{P}_{\text{soft-free}})\). Given that \(\text{dim}(\mathcal{N}(\mathbf{P}_{\text{soft-free}})) \ge N - d_k\), the majority of high-frequency variations in high-dimensional space tend to reside within the component \(\mathbf{x}_{\perp}\).

For \(\mathbf{P}_{\text{soft-free}}\), computing the output signal \(\mathbf{y}_{\text{soft-free}} = \mathbf{P}_{\text{soft-free}}\mathbf{x}\) yields:
\begin{equation}
    \mathbf{y}_{\text{soft-free}} = \mathbf{P}_{\text{soft-free}}(\mathbf{x}_{\parallel} + \mathbf{x}_{\perp}) = \mathbf{P}_{\text{soft-free}}\mathbf{x}_{\parallel} + \mathbf{0},
\end{equation}
\(\mathbf{P}_{\text{soft-free}}\) filters components spanning \(N-d_k\) dimensions. Since high-frequency components are typically orthogonal to the principal directions (primarily exists in \(\mathbf{x}_{\perp}\)), their Dirichlet energy is nullified, resulting in the loss of high-frequency information: \(E_{\text{Dirichlet}}(\mathbf{P}_{\text{soft-free}}\mathbf{x}_{\perp}) = \mathbf{0}^T \mathbf{L} \mathbf{0}\) (as shown in fig. \ref{fig:rank restoration}).

In the case of \(\mathbf{P}_{\text{soft}}\), its full-rank property ensures that no non-zero signal is mapped to the zero vector. Therefore, for the high-frequency component \(\mathbf{x}_{\perp}\), we have \(\mathbf{y}_{\text{soft}} = \mathbf{P}_{\text{soft}}\mathbf{x}_{\perp} \neq \mathbf{0}\). Since the output remains non-zero and \(\mathbf{P}_{\text{soft}}\) retains its signal transformation capability, the output signal \(\mathbf{y}_{\text{soft}}\) preserves a non-zero Laplacian Quadratic Form response:
\begin{equation}
    E_{\text{Dirichlet}}(\mathbf{P}_{\text{soft}}\mathbf{x}_{\perp}) = (\mathbf{P}_{\text{soft}}\mathbf{x}_{\perp})^T \mathbf{L} (\mathbf{P}_{\text{soft}}\mathbf{x}_{\perp}) > 0.
\end{equation}
Consequently, the softmax operation mitigates the loss of high-frequency information to a certain extent.

\end{proof}

\subsubsection{Proof of \texorpdfstring{\cref{thm: softmax-free spiking self attention}}{Theorem 3.5}}
\label{appendix: proof of theorem 3.5}

\begin{definition}
\label{def: graph fourier transform}
(Graph Fourier Transform) We define the normalized Laplacian matrix of an undirected weighted graph as \(\mathcal{L} = I - D^{-1/2}A_{sym}D^{-1/2}\). The eigenvectors \(\{u_l\}_{l=1}^N\) of \(\mathcal{L}\) constitute the Graph Fourier Basis, while the corresponding eigenvalues \(0 \le \lambda_1 \le \lambda_2 \dots \le \lambda_N\) represent the frequencies. In this context, low frequencies correspond to signals that vary smoothly across adjacent nodes, whereas high frequencies correspond to signals exhibiting rapid fluctuations between neighboring nodes.
\end{definition}

\begin{customthm}{3.5}
\label{thm: spiking self attention}
SSA \cite{zhouspikformer}, SDSA \cite{yao2023spike}, and Meta-SDSA \cite{yaospike} behave as low-pass filters in the frequency domain.
\end{customthm}

\begin{proof}
The operators \(\mathcal{A}\) corresponding to SSA, SDSA, and Meta-SDSA are shown in the table \ref{tab:unified spiking self attention representation}. Below, we will elaborate on their frequency-domain properties from the perspective of Graph Signal Processing (GSP). In the standard formulation of GSP, the output \(O\) is obtained by applying the adjacency matrix \(\mathcal{A}\) to the input \(V\): \(O=\mathcal{A} \cdot V\).

\textbf{a. For SSA and Meta-SDSA}. SSA and Meta-SDSA are represented in the forms given in Equ. \ref{equ: SSA} and Equ. \ref{equ: Meta-SDSA}, respectively. For SSA, the attention map is first computed as:
\begin{equation}
    O = \underbrace{(\mathbf{Q} \mathbf{K}^\top)}_{\text{Explicit } \mathcal{A}} \mathbf{V},
\end{equation}
where \(\mathcal{A}\) is an explicitly computed \(N \times N\) matrix. Therefore, the adjacency matrix \(\mathcal{A}\) from the GSP perspective is \(\mathbf{Q}\mathbf{K}^T\).

Meta-SDSA exploits the associativity of matrix multiplication to alter the computation sequence: \(O = \mathbf{Q} (\mathbf{K}^\top \mathbf{V})\). By invoking the associative property to revert this equation to the standard GSP formulation, \(\mathbf{Q} (\mathbf{K}^\top \mathbf{V}) = (\mathbf{Q} \mathbf{K}^\top) \mathbf{V}\). Consequently, the equivalent adjacency matrix defined within this system remains: \(\mathcal{A} = \mathbf{Q} \mathbf{K}^T\).

Therefore, from the GSP perspective, Meta-SDSA merely modifies the computational path (reducing complexity from \(O(N^2)\) to \(O(N)\)) without altering the underlying graph connectivity definition. That is, the theoretical connection weight between node i and node j remains defined by the similarity between \(\mathbf{Q}_i\) and \(\mathbf{K}_j\).

\textit{\textbf{Step 1}}. Let \(\mathbf{Q}, \mathbf{K} \in \mathbb{R}^{N \times d}\), where \(N\) denotes the sequence length and d represents the feature dimension. Typically, in Transformer architectures, \(d \ll N\). The rank of the operator \(\mathcal{A}\) satisfies:
\begin{equation}
    \text{rank}(\mathcal{A}) \le \min(N, d) = d.
\end{equation}
This implies that \(\mathcal{A}\) possesses at most d non-zero eigenvalues.

\textit{\textbf{Step 2}}. Assuming \(\mathbf{Q}\) and \(\mathbf{K}\) are approximately identically distributed, \(\mathcal{A} \approx \mathbf{Q} \mathbf{Q}^T\) constitutes a symmetric Positive Semi-Definite matrix. We perform the eigen-decomposition of \(\mathcal{A}\):
\begin{equation}
    \mathcal{A} = U \Lambda_{\mathcal{A}} U^\top = \sum_{i=1}^{N} \mu_i u_i u_i^\top,
\end{equation}
where \(\mu_1 \ge \mu_2 \ge \dots \ge \mu_N\) are the eigenvalues of \(\mathcal{A}\). Given the rank constraint \(\text{rank}(\mathcal{A}) \le d\), it follows that: \(\mu_i = 0, \quad \forall i > d\). Consequently, the spectral response function is defined as:
\begin{equation}
    \hat{h}(\lambda_i) = \begin{cases} \mu_i & \text{if } 1 \le i \le d \\ 0 & \text{if } i > d \end{cases}.
\end{equation}

\textit{\textbf{Step 3}}. Transforming the input \(\mathbf{V}\) into the frequency domain and filtering via \(\mathcal{A}\) yields:
\begin{equation}
    \mathbf{O} = \mathcal{A} \cdot V = \sum_{i=1}^{d} \mu_i (u_i^\top V) u_i + \sum_{i=d+1}^{N} 0 \cdot (u_i^\top V) u_i.
\end{equation}

The equation above demonstrates that \(\mathcal{A}\) effectively filters out high-frequency components corresponding to \(i > d\). Therefore, in contrast to softmax attention, the softmax-free attention operator \(\mathcal{A}=\mathbf{Q}\mathbf{K}^T\), when under the perspective of GSP, functions as a low-pass filter.

\textbf{b. For SDSA.} SDSA is expressed in the form of Equ. \ref{equ: SDSA}. In the computation of SDSA, the Hadamard product of \(\mathbf{Q}\) and \(\mathbf{K}\) can be viewed as a masking operation:
\begin{equation}
    \text{SDSA}(Q, K, V) = \underbrace{(\mathbf{Q} \otimes \mathbf{K})}_{\text{Mask } M} \otimes \mathbf{V},
\end{equation}
where M = \(Q \otimes K\) represents a generated Mask vector (or a weight vector in \(\mathbb{R}^{N \times 1}\) following channel-wise summation). For simplification, let us assume \(M\) is a vector \(m \in \mathbb{R}^N\) of length \(N\), where \(m_i\) denotes the importance weight of the \(i\)-th token. Consequently, the calculation for the \(i\)-th row of the output \(O\) (corresponding to the \(i\)-th node) is given by \(O_i = m_i \cdot V_i\), implying that the \(i\)-th output is strictly the \(i\)-th input \(V_i\) scaled by a scalar coefficient \(m_i\).

In the matrix multiplication, row-wise weighting of a matrix \(V\) by a vector \(m\) is equivalent to left-multiplication by a diagonal matrix constructed from \(m\).

We define the diagonal matrix \(\mathcal{A} \in \mathbb{R}^{N \times N}\) such that its diagonal elements correspond to \(m\), with all off-diagonal elements set to zero:
\begin{equation}
    \mathcal{A} = \text{diag}(m) = \begin{bmatrix} m_1 & 0 & \dots & 0 \\ 0 & m_2 & \dots & 0 \\ \vdots & \vdots & \ddots & \vdots \\ 0 & 0 & \dots & m_N \end{bmatrix}.
\end{equation}

Performing the matrix multiplication \(\mathcal{A} V\) yields:
\begin{equation}
    (\mathcal{A} V)_i = \sum_{j=1}^{N} \mathcal{A}_{ij} V_j = \mathcal{A}_{ii} V_i = m_i V_i,
\end{equation}

This aligns perfectly with the original SDSA formulation. From the GSP perspective, SDSA is equivalent to left-multiplication by a diagonal matrix. The adjacency matrix entry \(\mathcal{A}_{ij}\) describes the intensity of the influence of node \(j\) on node \(i\). Thus, when \(\mathcal{A}\) is a diagonal matrix: for \(i = j\): \(\mathcal{A}_{ii} = Q_i \cdot K_i\); for \(i \neq j\): \(\mathcal{A}_{ij} = 0\).

Let the input signal \(V\) consist of a valid signal component \(V_{\text{signal}}\) and high-frequency noise \(\xi_{\text{noise}}\): \(V = V_{\text{signal}} + \xi_{\text{noise}}\), where \(\xi_{\text{noise}}\) manifests as random high-frequency perturbations across the entire graph.

In the frequency domain, \(\mathcal{A}=\text{diag}(M(\textbf{Q},\textbf{K}))\) is equivalent to convolving the spectrum \(\hat{V}(\omega)\) of the signal \(V\) with a gating spectrum \(\hat{M}(\omega)\). Since \(M\) approaches 1 only at sparse, strongly correlated positions (\textit{i.e.}, where \(\mathbf{Q}\) matches \(\mathbf{K}\)), it acts as a narrow-band filter:
\begin{itemize}
    \item For the noise \(\xi_{\text{noise}}\): Since the noise is randomly distributed and \(M\) is extremely sparse, \(\text{diag}(M) \cdot \xi_{noise} \approx 0\). SDSA effectively eliminates the majority of spatial high-frequency noise by severing connections.
    \item For the signal \(V_{\text{signal}}\): \(M\) acts as an open gate only when signal features are significant (\(\mathbf{Q}\) matches \(\mathbf{K}\)).
\end{itemize}

Therefore, SDSA is equivalent to a low-pass filter.
\end{proof}

\subsubsection{Proof of \texorpdfstring{\cref{thm: differential}}{Theorem 4.1}}
\label{appendix: proof of theorem 4.1}

\begin{customthm}{4.1}
Introducing a differential mechanism into attention enhances high-frequency information. Furthermore, the differential mechanism computes attention scores as the difference between two separate attention maps, thereby amplifying attention to relevant context.
\end{customthm}

\begin{proof}
\textbf{a. Frequency Response and High-frequency Enhancement.} Let \(\mathbf{A}_1\) and \(\mathbf{A}_2\) denote two attention branches in the absence of a differential mechanism. We define the general form of differential attention as:
\begin{equation}
    \mathbf{A}_{\text{diff}} = \mathbf{A}_1 - \gamma \mathbf{A}_2, \ \gamma \in (0,1].
\end{equation}
Given the context of identical input, the spectral structures of the two attention branches are highly correlated. Thus, assuming \(\mathbf{A}_1\) and \(\mathbf{A}_2\) share an approximate eigenbasis \(\mathbf{U}\), the spectral response (eigenvalues) of \(\mathbf{A}_{\text{diff}}\) is given by:
\begin{equation}
    \lambda_k(\mathbf{A}_{\text{diff}}) \approx \lambda_k(\mathbf{A}_1) - \gamma \lambda_k(\mathbf{A}_2).
\end{equation}

In the low-frequency regime (\(k=1\)), \(\lambda_{DC}(\mathbf{A}_{diff}) = 1 - \gamma \cdot 1 = 1 - \gamma\). As \(\gamma \to 1\), \(\lambda_1(\mathbf{A}_{diff}) \to 0\). This implies that the differential mechanism eliminates the dominant low-frequency/DC component, preventing information from collapsing towards the mean.

In the high-frequency regime (\(k \ge 1\)), let \(\mathbf{A}_1\) be the attention branch focusing on local features (sharper, slower high-frequency decay), and \(\mathbf{A}_2\) be the branch focusing on global background (smoother, faster high-frequency decay). Consequently, the following inequality holds in the high-frequency band:
\begin{equation}
    |\lambda_k(\mathbf{A}_1)| > |\lambda_k(\mathbf{A}_2)| \approx 0.
\end{equation}
Thus, the high-frequency response of the differential operator is: \(\lambda_k(\mathbf{A}_{\text{diff}}) \approx \lambda_k(\mathbf{A}_1) - 0 = \lambda_k(\mathbf{A}_1)\).

In summary, the transfer function \(h(\lambda)\) of the differential operator exhibits band-pass or high-pass characteristics:
\begin{equation}
    |h_{\text{diff}}(\lambda)| \approx \begin{cases}
    0 & \substack{\text{if } \lambda \approx \lambda_{DC}\\
    \quad (\text{Global smoothing suppression})}\\
    |\lambda(\mathbf{A}_1)| & \substack{\text{if } \lambda \in \text{high frequency}\\
    \quad (\text{Detail preservation})}
    \end{cases}
\end{equation}

\textbf{b. Contextual SNR Amplification via Common-Mode Noise Elimination.} Assume the attention distribution is composed of a sparse relevant context signal \(\mathbf{S}\) and dense isotropic noise \(\mathbf{N}\). The differential mechanism can significantly boost the output Signal-to-Noise Ratio (SNR) compared to single-head attention by eliminating Common-Mode Noise. The proof is as follows:

We decompose the attention map into signal and noise components: \(\mathbf{A} = \mathbf{A}^{\text{signal}} + \mathbf{A}^{\text{noise}}\), where \(\mathbf{A}^{\text{signal}}\) is a sparse matrix possessing significant values \(s_{ij}\) only at relevant Query-Key pairs \((i,j) \in \mathcal{E}_{rel}\). Conversely, \(\mathbf{A}^{\text{noise}}\) is a dense matrix representing the Probability Leakage induced by the softmax operation, assuming it approximates a uniform distribution \(\frac{1}{N}\mathbf{1}\mathbf{1}^\top\) or smooth Gaussian noise.

Consider two attention heads \(\mathbf{A}_1\) and \(\mathbf{A}_2\). Let \(\mathbf{A}_1\) capture the primary signal while retaining background noise: \(\mathbf{A}_1 = \mathbf{S} + \mathbf{N}\). Let \(\mathbf{A}_2\) capture smoother global dependencies, comprised primarily of noise (or serving as a smooth signal baseline): \(\mathbf{A}_2 = \epsilon \mathbf{S} + \mathbf{N}\). Here, \(\epsilon \in [0, 1)\) represents the sensitivity coefficient of \(\mathbf{A}_2\) to the signal, typically satisfying \(\epsilon \ll 1\).

Defining the SNR as the ratio of signal energy to noise energy, the SNR for \(\mathbf{A}_1\) is given by \(\text{SNR}_{\text{orig}} = \frac{||\mathbf{S}||_F}{||\mathbf{N}||_F}\). For the differential mechanism, we compute \(\mathbf{A}_{\text{diff}} = \mathbf{A}_1 - \gamma \mathbf{A}_2\):
\begin{equation}
    \mathbf{A}_{diff} = (\mathbf{S} + \mathbf{N}) - \gamma (\epsilon \mathbf{S} + \mathbf{N}) = (1 - \gamma \epsilon)\mathbf{S} + (1 - \gamma)\mathbf{N}.
\end{equation}

In this formulation, the effective signal component is \((1 - \gamma \epsilon)\mathbf{S}\), and the residual noise component is \((1 - \gamma)\mathbf{N}\). Consequently, the SNR of the differential operator is:
\begin{equation}
\begin{split}
\text{SNR}_{\text{diff}} &= \frac{||(1 - \gamma \epsilon)\mathbf{S}||_F}{||(1 - \gamma)\mathbf{N}||_F}\\
&= \frac{1 - \gamma \epsilon}{1 - \gamma} \cdot \frac{||\mathbf{S}||_F}{||\mathbf{N}||_F}\\
&= \underbrace{\frac{1 - \gamma \epsilon}{1 - \gamma}}_{\text{Amplification Factor }\kappa} \cdot \text{SNR}_{orig}.
\end{split}
\end{equation}

Examining the amplification factor \(\kappa(\gamma)\), when \(\gamma\) approaches 1 and \(\epsilon\) is small (implying \(\mathbf{A}_2\) is predominantly noise):
\begin{equation}
    \lim_{\gamma \to 1^-} \kappa = \lim_{\gamma \to 1^-} \frac{1 - \epsilon}{1 - \gamma} \to \infty.
\end{equation}
Even in practical settings where \(\gamma\) is a fixed value (e.g., 0.9), provided that \(\epsilon < 1\) (i.e., \(\mathbf{A}_1\) contains more valid information than \(\mathbf{A}_2\)), the condition \(\kappa > 1\) holds universally.

In summary, the differential mechanism eliminates the background noise shared by \(\mathbf{A}_1\) and \(\mathbf{A}_2\) (Common-Mode Noise), thereby mathematically guaranteeing \(\text{SNR}_{\text{diff}} > \text{SNR}_{\text{orig}}\). This effectively amplifies the attention weights assigned to relevant contexts.

\end{proof}

\subsection{Experimental Details}
\label{appendix: experimental details}

This section provides details on the experiments: the image classification experiment (ImageNet-1K) was conducted on 8 NVIDIA H100 GPUs, while the semantic segmentation (AED20K) and event-based tracking experiments were conducted on 8 NVIDIA 4090 GPUs.

\subsubsection{Image Classification}

All experiments were conducted using the PyTorch framework on the ImageNet-1K dataset \cite{deng2009imagenet}. The hyperparameters for each architecture are shown in the Table \ref{tab: hyperpara of models}.

\begin{table*}[htbp]
    \centering
    \caption{Comparison of Hyperparameters for Different Model Architectures}
    \label{tab: hyperpara of models}
    % lccc 表示第一列左对齐，后三列居中对齐
    \begin{tabular}{lccc}
        \toprule
        \textbf{Hyper-parameter} & \textbf{SDT-V1} & \textbf{QKformer} & \textbf{SDT-V3} \\
        \midrule
        Timestep & 4 & 4 & 4 \\
        Epochs & 300 & 200 & 200 \\
        Resolution & 224 & 224 & 224 \\
        Batch size & 48 & 100 & 600 \\
        Optimizer & AdamW & AdamW & LAMB \\
        Base Learning rate & 1.5e-5 & 6e-4 & 6e-4 \\
        Learning rate decay & Cosine & Layer-wise 1.0 & Layer-wise 1.0 \\
        Warmup epochs & 20 & 5 & 10 \\
        Weight decay & 0.01 & 0.05 & 0.05 \\
        Rand Augment & rand-m5-mstd0.5-inc1 & rand-m9-mstd0.5-inc1 & rand-m9-mstd0.5-inc1 \\
        Mixup & 0.2 & 0 & 0 \\
        Cutmix & 1.0 & 0 & 0 \\
        Label smoothing & 0.1 & 0.1 & 0.1 \\
        \bottomrule
    \end{tabular}
\end{table*}

To ensure fair comparisons between architectures, we strictly adhered to the open-source network architectures of Spike-driven Transformer V1 \cite{yao2023spike}, QKFormer \cite{zhou2024qkformer}, and Spike-driven Transformer V3 \cite{yao2025scaling}. Crucially, we only replaced the self-attention computation module within each architecture while keeping all other components unchanged. This controlled experimental design enabled us to isolate the impact of different attention mechanisms on model performance.

\begin{figure*}[t]
    \centering    
    \includegraphics[width=\linewidth]{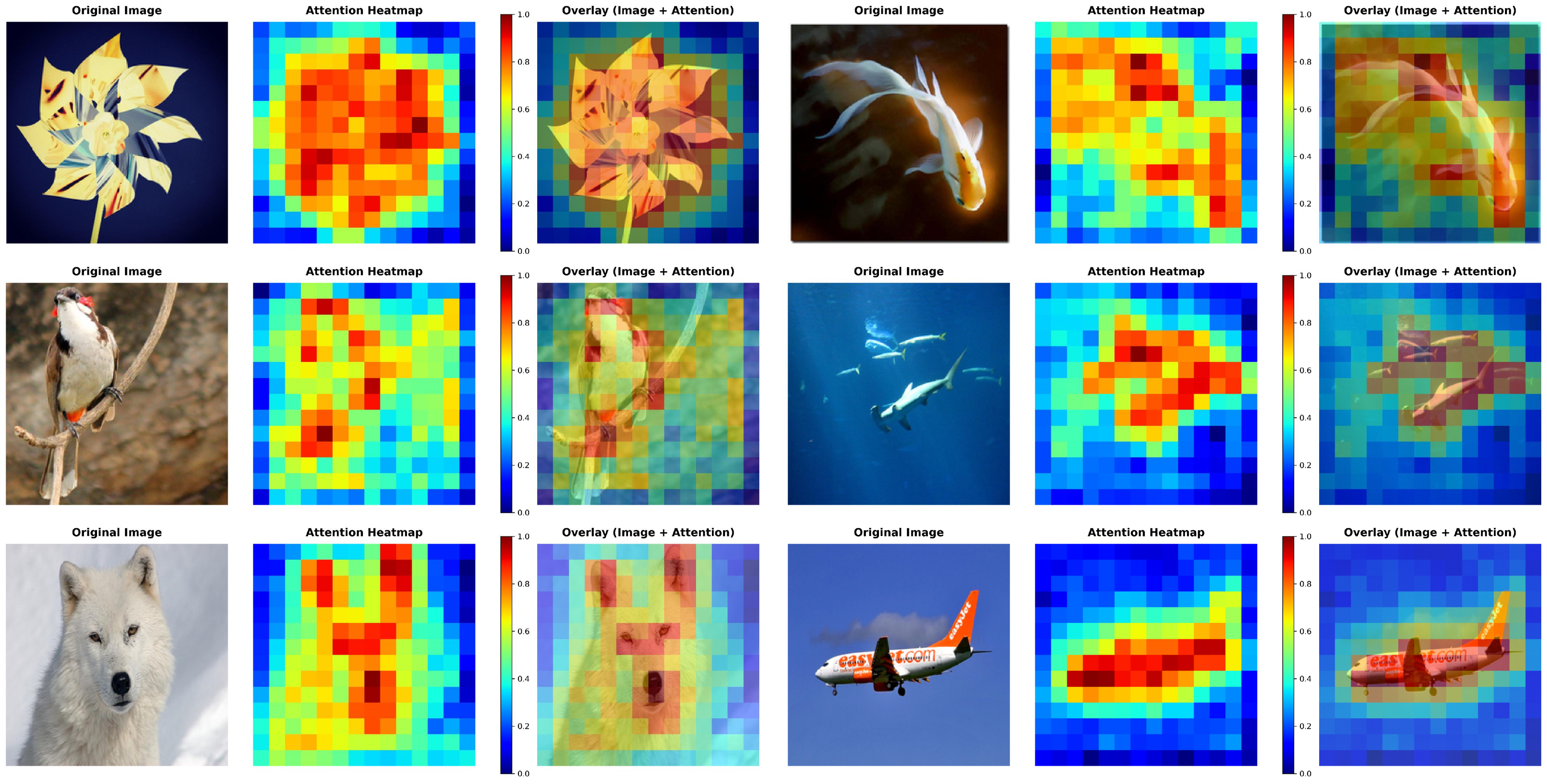}  % 设置页面宽度
    \caption{Attention heatmap for SCA (Ours) with SDT-V1 architecture on ImageNet-1K.}
    \label{fig:attention_heatmap_full}
\end{figure*}

\subsubsection{Semantic Segmentation}

This study employs the ADE20K semantic segmentation dataset \cite{zhou2017scene}, comprising over 20,000 training and 2,000 validation scene center images with meticulous pixel-level annotations of objects and their components. The dataset exhibits rich semantic diversity, encompassing 150 semantic categories covering environmental elements (sky, roads, grass, etc.) and discrete entities (people, vehicles, furniture, etc.).

For our experimental architecture, we utilize SDT-V3, pre-trained on ImageNet-1K, as the backbone network, integrated with PVT (Pyramid Vision Transformer) \cite{xie2021segformer,wang2021pyramid} for the semantic segmentation task. During model training, we set the batch size to 20 and the total number of iterations to 160,000. Our optimization strategy employs the AdamW optimizer with an initial learning rate of $1 \times 10^4$ and a polynomial decay of 0.9. Notably, we applied linear decay warm-up during the initial 150,000 iterations to enhance model stability. Finally, our work demonstrates significant performance improvements over the standard SDT-V3 architecture. Additionally, we visualize qualitative results of the model on the ADE20K dataset. As shown in the fig. \ref{fig:segmentation_results_full}, our segmentation results exhibit remarkable effectiveness, featuring precise boundary delineation and enhanced semantic consistency across diverse scene contexts.

\subsubsection{Event-based Tracking}

We propose an efficient event-based visual tracker based on the SDTrack pipeline \cite{shan2025sdtrack}, which locates targets by matching image pairs through visual feature comparisons between reference and current frames. Dataset details include: FE108 \cite{zhang2021object}: 108 event sequences (3000–5000 frames each) with annotated bounding boxes, diverse scenes, optimized for high-speed tracking; VisEvent \cite{wang2023visevent}: 60 RGB event-synchronized sequences (\(\sim\) 250,000 frames total), varying indoor/outdoor conditions, cross-modal robustness benchmark; FELT \cite{wang2024long}: 200 sequences, event accumulation within ultra-short windows (1--5 ms), challenging real-time scenarios (motion blur, occlusion, fast motion). The training schedule comprises 100 epochs for FE108 and VisEvent, and 300 epochs for FELT. Each epoch randomly samples 60k image pairs (max interval 200 frames) for FE108 and FELT, and 30k pairs for VisEvent, ensuring sample diversity. Optimization uses AdamW (initial LR=\(4 \times 10^{-4}\), decaying to \(4 \times 10^{-5}\) at 80\% epoch, weight decay=\(1 \times 10^{-4}\)).

\subsubsection{Source of Gain}
\label{appendix: source of gain}

If SCA's advantage merely comes from a generic architectural inductive bias, it should provide a uniform performance boost regardless of the input's frequency distribution. If, however, our core claim holds true—that SCA specifically recovers high-frequency information—its advantage over the baseline should amplify significantly on inputs dominated by high-frequency features, and diminish on inputs lacking high-frequency details.

\begin{figure*}[htbp]
    \centering    
    \includegraphics[width=\linewidth]{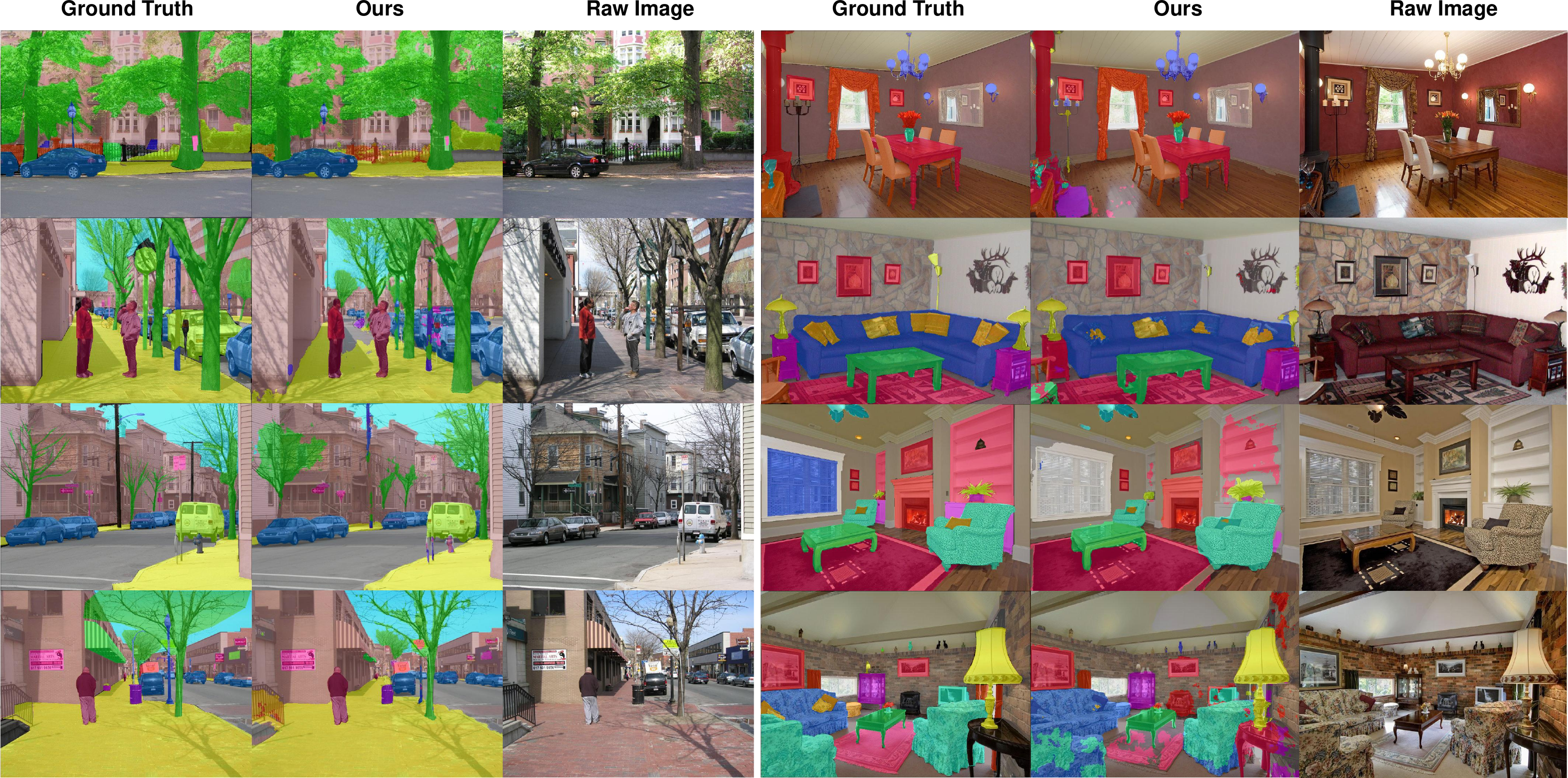}  % 设置页面宽度
    \caption{Visualization of semantic segmentation results on the ADE20K dataset.}
    \label{fig:segmentation_results_full}
\end{figure*}

To prove this, we processed the CIFAR100 \cite{krizhevsky2009learning} test-set using a High-Pass filter (retaining only sharp edges/textures) and a Low-Pass filter (Gaussian blur, removing all high frequencies). We evaluated the pre-trained QKFormer and QKFormer+SCA on these corrupted datasets.

\begin{table*}[htbp]
    \centering
    \caption{Performance comparison under different test-set conditions}
    \label{tab:performance comparision under different test-set conditions}
    \begin{tabular}{lccc}
        \toprule
        \textbf{Test Set Condition} & \textbf{QKFormer Baseline} & \textbf{QKFormer + SCA} & $\mathbf{\Delta}$ \textbf{(Gain)} \\
        \midrule
        Standard Image (Full spectrum)        & $81.15 \pm 0.18$ & $82.34 \pm 0.12$          & $+1.19\%$ \\
        Low-Pass Filter (High-freq removed)   & $77.12 \pm 0.09$ & $77.18 \pm 0.13$          & $+0.06\%$ \\
        High-Pass Filter (High-freq dominant) & $68.45 \pm 0.21$ & $\mathbf{74.80 \pm 0.16}$ & $\mathbf{+6.35\%}$ \\
        \bottomrule
    \end{tabular}
\end{table*}

As the Table.\ref{tab:performance comparision under different test-set conditions} clearly demonstrates, when high-frequency information is removed (Low-Pass), SCA's architectural inductive bias provides virtually zero gain (\(+0.06\)). However, when the task strictly relies on high-frequency edges (High-Pass), SCA massively outperforms the baseline (\(+6.35\%\)). This provides unambiguous, statistically robust evidence that SCA's performance leap is specifically attributable to its ability to preserve and rectify high-frequency spectral bias, precisely as theorized.

\subsubsection{Generality across Convolutional Architectures}
\label{appendix: generality across convolutional architectures}

By introducing a lightweight contrastive branch before the LIF activations in the residual blocks, we applied the core philosophy of SCA (the differential membrane-potential mechanism) to SEW-ResNet (Wide-7B-Net) and MS-ResNet. The results is shown as Table. \ref{tab:generality across convolutional architectures}. For the SEW-ResNet architecture, on the CIFAR-10 DVS dataset \cite{li2017cifar10}, this method improved the accuracy by 2.99\% at time step \(T=4\) and by 2.46\% at \(T=8\). For the MS-ResNet-18 architecture (at \(T=4\)), this method improved the accuracy by 1.97\% on the CIFAR-10 dataset \cite{lecun2002gradient} and by 2.09\% on the CIFAR-100 dataset \cite{krizhevsky2009learning}.

\begin{table*}[htbp]
    \centering
    \caption{Generality across Convolutional Architectures}
    \label{tab:generality across convolutional architectures}
    \resizebox{\textwidth}{!}{
    \begin{tabular}{lccccc}
        \toprule
        \textbf{Architecture} & \textbf{Params (M)} & \textbf{\(T\)} & \textbf{CIFAR-10 Acc. (\%)} & \textbf{CIFAR-100 Acc. (\%)} & \textbf{CIFAR-10 DVS Acc. (\%)} \\
        \midrule
        \multirow{2}{*}{SEW-ResNet \cite{fang2021deep}} & \multirow{2}{*}{1.19} & 4 & - & - & 64.80 \\
         & & 8 & - & - & 70.20 \\
        \midrule
        \multirow{2}{*}{SEW-ResNet (Ours)} & \multirow{2}{*}{1.19} & 4 & - & - & \textbf{67.79} \textcolor{violet}{(\(\uparrow\) 2.99)} \\
         & & 8 & - & - & \textbf{72.66} \textcolor{violet}{(\(\uparrow\) 2.46)} \\
        \midrule
        MS-ResNet-18 \cite{hu2024advancing} & 11.22 & 4 & 94.40 & 75.06 & - \\
        \midrule
        MS-ResNet-18 (Ours) & 11.22 & 4 & \textbf{96.37} \textcolor{violet}{(\(\uparrow\) 1.97)} & \textbf{77.15} \textcolor{violet}{(\(\uparrow\) 2.09)} & - \\
        \bottomrule
    \end{tabular}}
\end{table*}

\begin{figure*}[t]
    \centering    
    \includegraphics[width=0.92\linewidth]{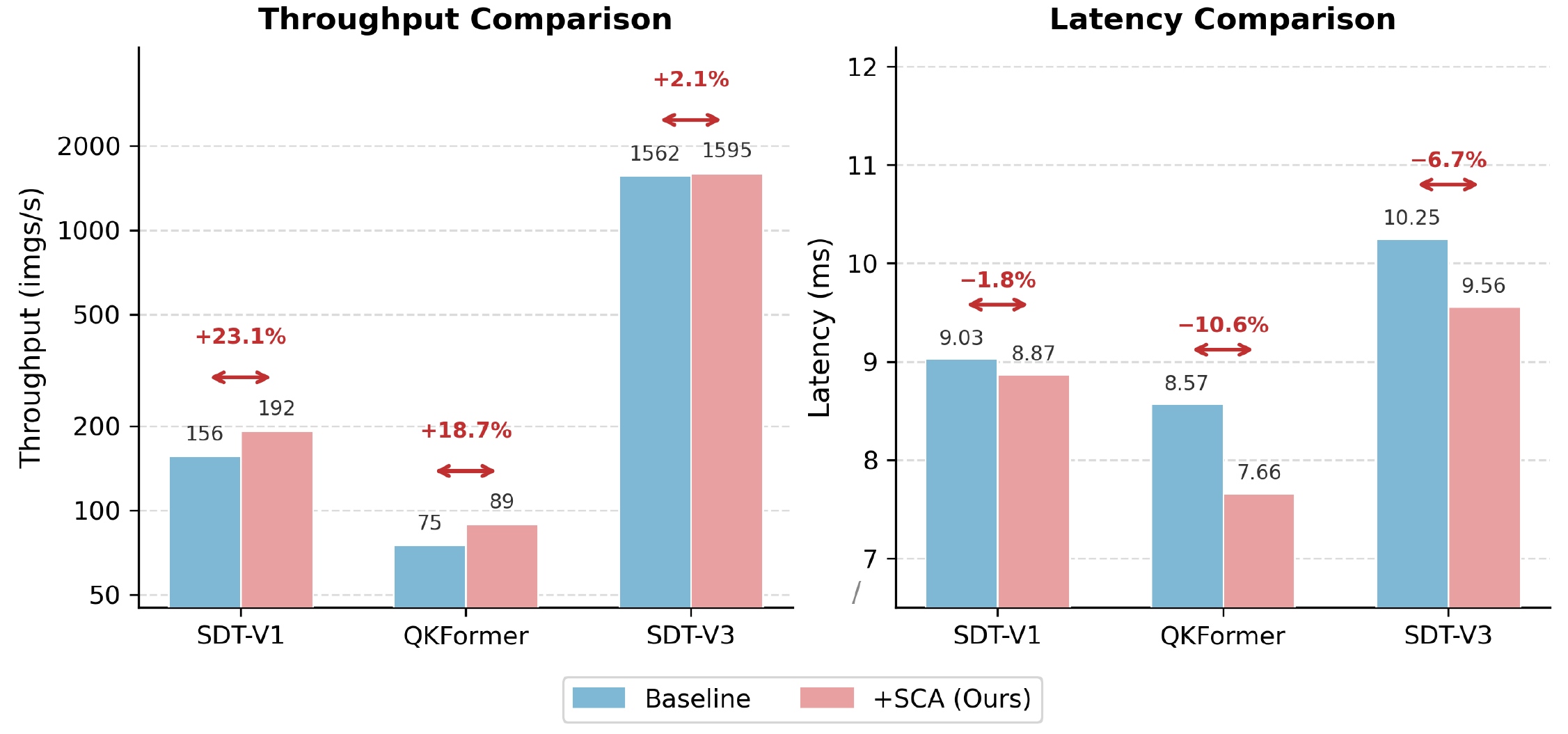}
    \caption{Throughput and latency comparison in baselines and our method.}
    \label{fig:thoughput and latency}
\end{figure*}

\subsubsection{Throughput and Inference Latency}
\label{appendix: throughput and inference latency}

To substantiate our claim of lower computational complexity, we measured the actual throughput and inference latency. The evaluation was conducted on a single NVIDIA RTX 4090 GPU with a batch size of 64. As shown in Fig.\ref{fig:thoughput and latency}, the theoretical complexity reduction from \(\mathcal{O}(TN^2C)\) to \(\mathcal{O}(TNnC)\) directly translates to measurable latency improvements.

\subsection{Limitations}
\label{appendix: limitations}
Limitations of this study include the lack of performance testing of SCA on larger-scale models, such as spike-based large language model, as well as limitations related to hardware deployment. These issues will be addressed in future research. In the future, we will evaluate SCA’s performance on larger-scale models and explore deploying SCA on hardware platforms such as Field-Programmable Gate Arrays (FPGAs) to assess its practical performance. In doing so, we will optimize appropriate data read/write streams and memory access schemes to enhance the model’s inference speed.

\end{document}